\documentclass[5p,authoryear]{elsarticle}
\usepackage{adjustbox}

\usepackage{times}
\usepackage{multicol}
\usepackage[bookmarks=true]{hyperref}
\usepackage{xcolor}
\usepackage{hyperref}
\usepackage{amsmath, amssymb}
\usepackage{amsfonts}
\usepackage{graphicx}
\usepackage{siunitx}
\usepackage{standalone}
\usepackage{booktabs}
\usepackage[ruled,vlined,linesnumbered]{algorithm2e}
\usepackage{mdframed}
\usepackage{fancyvrb,multirow}
\usepackage{soul}
\usepackage{dsfont,mathabx}
\usepackage{array, booktabs}
\usepackage{subfigure}
\usepackage{amsthm}
\usepackage{makecell}
\usepackage{url}

\newtheorem{theorem}{Theorem}
\newtheorem{proposition}{Proposition}
\newtheorem{lemma}[theorem]{Lemma}
\newtheorem{assumption}{Assumption}
\theoremstyle{definition}
\newtheorem{definition}{Definition}
\newtheorem{remark}{Remark}

\journal{Automatica}

\biboptions{authoryear}

\begin{document}

\begin{frontmatter}

\title{\LARGE Hybrid and Oriented Harmonic Potentials for Safe Task Execution\\
  in Unknown Environment}

\author{Shuaikang Wang and Meng Guo \fnref{pku}}
\address{Department of Mechanics and Engineering Science,\\
	College of Engineering, Peking University, Beijing 100871, China.}
\fntext[pku]{This work was supported by the National Natural Science Foundation
    of China (NSFC) under grants 62203017, T2121002, U2241214;
    and by the Fundamental Research Funds for the central universities.
    Contact: \texttt{wangs, meng.guo@pku.edu.cn}.}

\begin{abstract}
Harmonic potentials provide globally convergent potential fields
that are provably free of local minima.
Due to its analytical format, it is particularly suitable
for generating safe and reliable robot navigation policies.
However, for complex environments that consist of a large number
of overlapping non-sphere obstacles,
the computation of associated transformation functions can be tedious.
This becomes more apparent when: (i) the workspace is
initially unknown and the underlying potential fields are updated constantly
as the robot explores it;
(ii) the high-level mission consists of sequential navigation tasks among numerous regions,
requiring the robot to switch between different potentials.
Thus, this work proposes an efficient and automated scheme to construct
harmonic potentials incrementally online as guided by the task automaton.
A novel two-layer harmonic tree (HT) structure is introduced
that facilitates the hybrid combination of oriented search algorithms for task planning
and harmonic-based navigation controllers for non-holonomic robots.
Both layers are adapted efficiently and jointly during online execution to
reflect the actual feasibility and cost of navigation within the updated workspace.
Global safety and convergence are ensured both for the high-level task plan
and the low-level robot trajectory.
Known issues such as oscillation or long-detours for purely potential-based methods
and sharp-turns or high computation complexity for purely search-based methods are prevented.
Extensive numerical simulation and hardware experiments
are conducted against several strong baselines.
\end{abstract}

\end{frontmatter}

\section{Introduction}\label{sec:intro}
Autonomous robots can replace humans to operate and accomplish complex missions in hazardous environments.
However, it is a demanding engineering task to ensure both the safety and efficiency during execution,
especially when the environment is only partially known.
First, the control strategy
that drives the robot from an initial state to the goal state while staying within
the allowed workspace
(see e.g., \cite{karaman2011sampling,lavalle2006planning,koditschek1987exact,khatib1999mobile}),
should be reactive to the newly-discovered obstacles online.
Second, the planning method that decomposes and schedules
sub-tasks (see e.g., \cite{ghallab2004automated,fainekos2009temporal})
should be adaptive to the actual feasibility and cost of sub-tasks given the updated environment.
Existing work often ignores the close dependency of these two modules and treats them separately,
which can lead to inefficient or even unsafe executions,
as also motivated in~\cite{garrett2021integrated,kim2022representation}.
How to construct a fully integrated task and motion planning scheme with provable safety and efficiency
guarantee within unknown environments still remains challenging, see~\cite{loizou2022mobile, rousseas2022trajectory}.
\begin{figure*}[t]
  \centering
  \includegraphics[width=0.99\hsize]{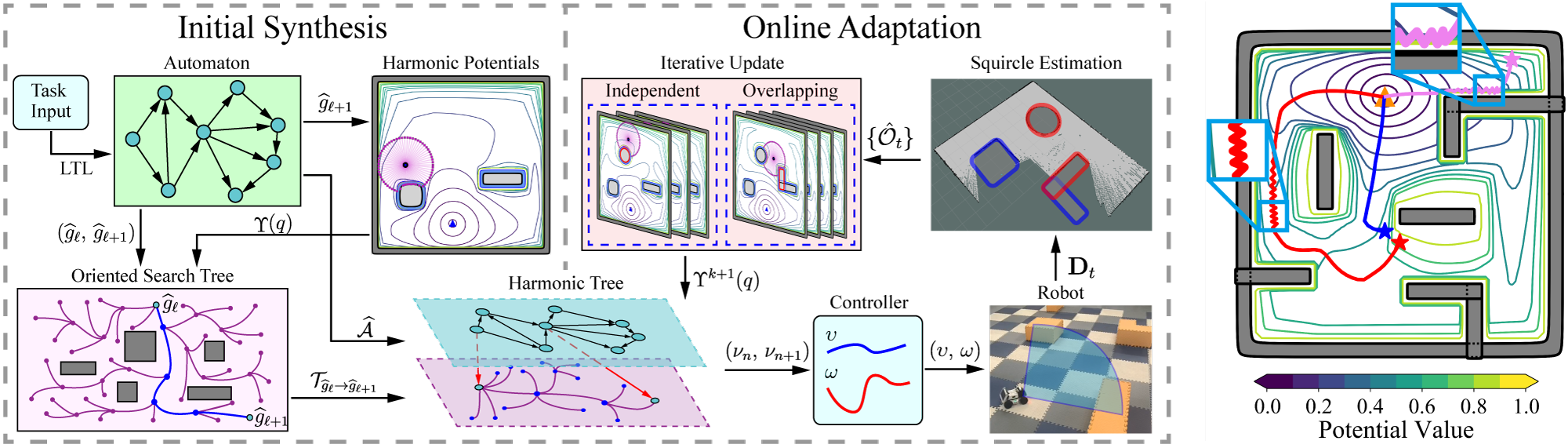}
  \vspace{-0.1in}
  \caption{\textbf{Left}: Illustration of the proposed framework,
    which consists of the initial synthesis, the online adaption of the task plan
    and harmonic potentials, and the squircle estimation.
  \textbf{Right}:
  Oscillations and long detours might occur via classic navigation functions
  as shown in the red, violet and blue trajectories.}
  \label{fig:diagram}
  \vspace{-0.05in}
\end{figure*}
\subsection{Related Work}\label{subsec:intro-related}
As the most relevant to this work, the method of artificial potential
fields from~\cite{khatib1986real, warren1989global, panagou2014motion, rousseas2022optimal} introduces
an intuitive yet powerful framework for tackling the safety and convergence property
during navigation.
The main idea is to introduce attractive potentials to the goal state
and repulsive potentials from obstacles and the workspace boundary.
However, naive design of these potentials would introduce undesired local minima,
where the combined forces are zero and thus prevents further progress.
Navigation functions (NF) pioneered by~\cite{koditschek1987exact}
provably guarantee that such minima are saddle points and more importantly of measure zero.
Although the underlying static workspace could be as general as \emph{forest of stars},
some key design parameters require fine-tuning for
the safety and convergence properties to hold.
The work in~\cite{fan2022robot} employs the conformal transformations to
map the multiply-connected workspaces to a sphere world without any tuning parameter,
which however requires a numerical solution of continuous integrals.
The work
in~\cite{huber2019avoidance,huber2024avoidance, dahlin2023creating} utilizes
reactive potentials to guide dynamic systems around obstacles.
Model predictive control has been adopted
in~\cite{dahlin2023obstacle, yu2015nonlinear, faulwasser2015nonlinear, sanchez2021nonlinear}
to track the derived potentials.
Moreover, harmonic potentials proposed in~\cite{kim1992real, loizou2011closed,
loizou2017navigation, vlantis2018robot}
alleviate such limitations by introducing a novel transformation scheme
from obstacle-cluttered environments to {point worlds},
while retaining these properties.
Furthermore, recent work in
~\cite{rousseas2021harmonic,
  rousseas2022trajectory,
  rousseas2022optimal} resolves the need
for a diffeomorphic mapping onto sphere disks,
by adopting a wider set of basis functions for workspace boundaries.
Nevertheless, such methods require solving numerous complex parametric optimizations,
instead of an analytic solution.
Lastly, despite of their global convergence guarantee,
there are several notable limitations as illustrated in Fig.~\ref{fig:diagram}:
(i) oscillations or jitters may appear especially when the trajectory slides along the
boundary of obstacles or crosses narrow passages;
(ii) drastically different trajectories may develop within the same potential field
when the initial pose is changed slightly;
(iii) the resulting trajectory is far from the optimal one in terms of
trajectory length or control efforts;
(iv) the final orientation at the goal pose can not be controlled freely.

Moreover, sampling-based search methods,
such as RRT$^\star$ in \cite{karaman2011sampling},  PRM in \cite{hsu2006probabilistic},
FMT$^\star$ in \cite{janson2015fast},
have become the dominant paradigm to tackle high-dimensional motion planning problems,
especially for systems under geometric and dynamic constraints.
However, one potential limiting factor is the high computational complexity
due to the collision checking process between sampled states
and the excessive sampling to reach convergence.
Since artificial potential fields are analytical
with theoretical guarantee,
it make sense to combine these two paradigms:
vector fields are used in~\cite{ko2013vf} to bias the branching of search trees,
thus improving the efficiency of sampling and reducing the number of iterations;
similar ideas are adopted in~\cite{qureshi2016potential,tahir2018potentially} as
the potential-based RRT$^\star$,
by designing directional samples as induced by the underlying potential fields.
However, these methods mostly focus on static environments for simple navigation tasks,
where the planning is performed offline and
neither the potential fields nor the search structure are adapted during execution.

When a robot is deployed in a partially-unknown environment,
an online approach is required such that the
underlying trajectory adapts to real-time measurements of the actual workspace
such as new obstacles.
For instances, a fully automated tuning mechanism for navigation functions
is presented in~\cite{filippidis2011adjustable},
while the notion of dynamic windows is proposed in~\cite{ogren2005convergent}
to handle dynamic environment.
Moreover, the harmonic potentials-based methods are developed further
in~\cite{rousseas2022trajectory} for unknown environments,
where the weights over harmonic basis are optimized online.
A similar formulation is adopted in~\cite{loizou2022mobile} where the parameters
in the harmonic potentials are adjusted online to ensure safety and global convergence.
{Furthermore, a semantic perceptual feedback me-thod is introduced in~\cite{vasilopoulos2022reactive} to
recognize the size of the obstacles from a pre-trained dataset.}
Lastly, the scenario of time-varying targets is analyzed in~\cite{li2018navigation}
by designing an attractive potential that evolves with time.
Similarly, dynamic environments are considered in~\cite{huber2022avoiding, dahlin2023creating}
by allowing time-varying and reactive potentials.
On the other hand, search-based methods are
also extended to unknown environments where various real-time revision techniques
are proposed in~\cite{otte2016rrtx, shen2021smarrt}.
However, less work can be found where the search tree and the underlying potentials
should be updated simultaneously and dependently.

Last but not least, the desired task for the robot could be more complex
than the point-to-point navigation.
Linear Temporal Logics (LTL) in \cite{baier2008principles} provide a formal language
to describe complex high-level tasks, such as sequential visit, surveillance and response.
Many recent papers can be found that combine robot motion planning
with model-checking-based task planning,
e.g., a single robot under LTL tasks~\cite{fainekos2009temporal, guo2018human, lindemann2021stl},
a multi-robot system under a global task~\cite{guo2015multi,luo2021abstraction,leahy2021scalable},
However, many aforementioned work assumes
an existing low-level navigation controller, or considers a simple and known
environment with circular and non-overlapping obstacles.
The synergy of complex temporal tasks and
harmonic potential fields within unknown environments has not been investigated.

\subsection{Our Method}\label{subsec:intro-our}
This work proposes an automated planning framework
term-ed that utilize harmonic potentials for navigation
and oriented search trees for planning, as illustrated in Fig.~\ref{fig:diagram}.
The design and construction of the search tree is specially tailored for the task automaton
and co-designed with the underlying navigation controllers based on harmonic potentials.
Intermediate waypoints are introduced between task regions to improve task efficiency
and smoothness of the robot trajectory.
Additionally,
a novel orientation-aware harmonic potential is proposed for nonholo-nomic robots,
based on which a nonlinear tracking controller is utilized to ensure safety.
Furthermore, during online execution, as the robot explores the environment gradually,
an efficient adaptation scheme is proposed to update the estimated obstacles, the search tree and the harmonic
potentials simultaneously, where intermediate variables are saved and re-used.
For validation, extensive simulations and hardware experiments are conducted for nontrivial tasks.

Main contribution of this work lies in the hybrid framework that combines
two powerful methods in control and planning,
for non-holonomic robots to accomplish complex tasks in unknown environments.
Specifically, it includes:
(i) the two-layer and automaton-guided Harmonic trees
that unify task planing and motion control;
(ii) a new ``purging'' method for forests of overlapping squircles,
which is tailored for the online case where obstacles are added gradually;
(iii) an integrated method to update the estimated obstacles,
  the Harmonic potentials
and the search trees simultaneously and recursively online.
It has been shown via both theoretical analyses and numerical studies that
it avoids the common problem of oscillation or long-detours for purely potential-based methods,
and sharp-turns or high computation complexity for purely search-based methods.

\subsection{Note for Practitioners}\label{subsec:practitioner}
To apply the proposed method to practical systems,
  the following steps are recommended given the framework in Fig.~\ref{fig:diagram}:
(i) the workspace model should be constructed w.r.t. the
specified task including the regions of interest and
their properties, as modeled in Sec.~\ref{subsec:ltl};
(ii) the abstraction method and relative distance
for the vertices within the harmonic tree
should be chosen according to the characteristics of the workspace, such as size and typical structure,
as described in Sec.~\ref{subsubsec:init-tree};
(iii) the nonlinear tracking controller
in Sec.~\ref{subsubsec:robot-control} is tuned for the specific
hardware platform such that it can track the gradient
of the oriented harmonic potentials
in Sec.~\ref{subsubsec:init-nf} with a desired accuracy;
(iv) the segmentation and clustering of the online data points
for the obstacle estimation
should be adjusted according the range and resolution of the
Lidar sensor, as mentioned in Sec.~\ref{subsubsec:init-nf};
(v) the update rate of the estimated obstacles
and the condition for replanning
should be tuned according to the estimated density of obstacles, as discussed in Sec.~\ref{subsec:online}.

\section{Preliminaries}\label{sec:preliminary}
\subsection{Diffeomorphic Transformation and Harmonic Potentials}\label{subsec:diff-transform}

A 2D \emph{sphere world}~$\mathcal{M}$ is defined as a compact and connected subset of~$\mathbb{R}^2$,
which has an outer boundary~$O_0=\{q \in \mathbb{R}^2: \|q-q_0\|^2-\rho_0^2 \leq 0\}$
centered at~$q_0$ with radius~$\rho_0$,
and inner boundaries of~$M$ disjoint sphere obstacles~$O_i =\{q \in \mathbb{R}^2: \|q-q_i\|^2-\rho_i^2 \leq 0 \}$
centered at~$q_i$ with radius~$\rho_i$, for $i=1,\cdots,M$.
There is a goal point denoted by~$q_G\in \mathcal{M}$.
By using a diffeomorphic transformation proposed in~\cite{loizou2017navigation},
this sphere world can be mapped to an unbounded \emph{point world}, as shown in Fig.~\ref{fig:diffeo-trans}.
It is denoted by {$\mathcal{P}=\mathbb{R}^2\backslash\{P_1,\cdots, P_M\}$},
which consists of $M$ point obstacles~$P_i\in \mathbb{R}^2$.
Specifically, the diffeomorphic transformation~$\Phi_{\mathcal{M}\rightarrow \mathcal{P}}(q)$
from sphere world to point world is constructed as follows:
\begin{equation}\label{eq:exact-tf-m-p}
  \begin{aligned}
    &\Phi_{\mathcal{M}\rightarrow \mathcal{P}}(q)\triangleq \psi \circ
    \Phi_{\mathcal{M} \rightarrow \Tilde{\mathcal{P}}}(q),\\
    &\Phi_{\mathcal{M}\rightarrow \Tilde{\mathcal{P}}}(q) \triangleq {\rm id}(q)
    +\sum_{i=1}^M \big(1-s_\delta(q,\,O_i)\big)(q_i-q),\\
    &\psi(\Tilde{q})  \triangleq \frac{\rho_0}{\rho_0-||\Tilde{q}-q_0||}(\Tilde{q}-q_0)+q_0,
    \end{aligned}
\end{equation}
where~$\Phi_{\mathcal{M}\rightarrow \Tilde{\mathcal{P}}}(q)$ transforms the sphere world~$\mathcal{M}$ to
a bounded point world {$\Tilde{\mathcal{P}}=O_0\backslash\{\Tilde{P}_1,\cdots, \Tilde{P}_M\}$},
of which~$\Tilde{P}_i$ is the inner point-shape obstacle with~$\Tilde{P}_i=\Phi_{\mathcal{M}\rightarrow \Tilde{\mathcal{P}}}(q_i)$,
for $i=1,\cdots,M$;
the summed element~$s_\delta(q,\,O_i)$ is the contraction-like transformation for obstacle~$O_i$,
which is composed by~$\eta_\delta(x) \circ \sigma(x) \circ b_i(x)$ as the
switch function, smoothing function and distance function, respectively.
{The exact definitions can be found in
\cite{loizou2017navigation} and the supplementary files}.
To obtain an infinite harmonic domain,
it is essential to map the bounded point world into the unbounded point world via
the diffeomorphic transformation~$\psi(\Tilde{q})$.
Given the point world~$\mathcal{P}$,
its associated~\emph{harmonic potential} function, is introduced
in~\cite{loizou2022mobile, loizou2021correct} and defined as follows.
\begin{figure}[t]
  \centering
  \includegraphics[width=0.95\hsize]{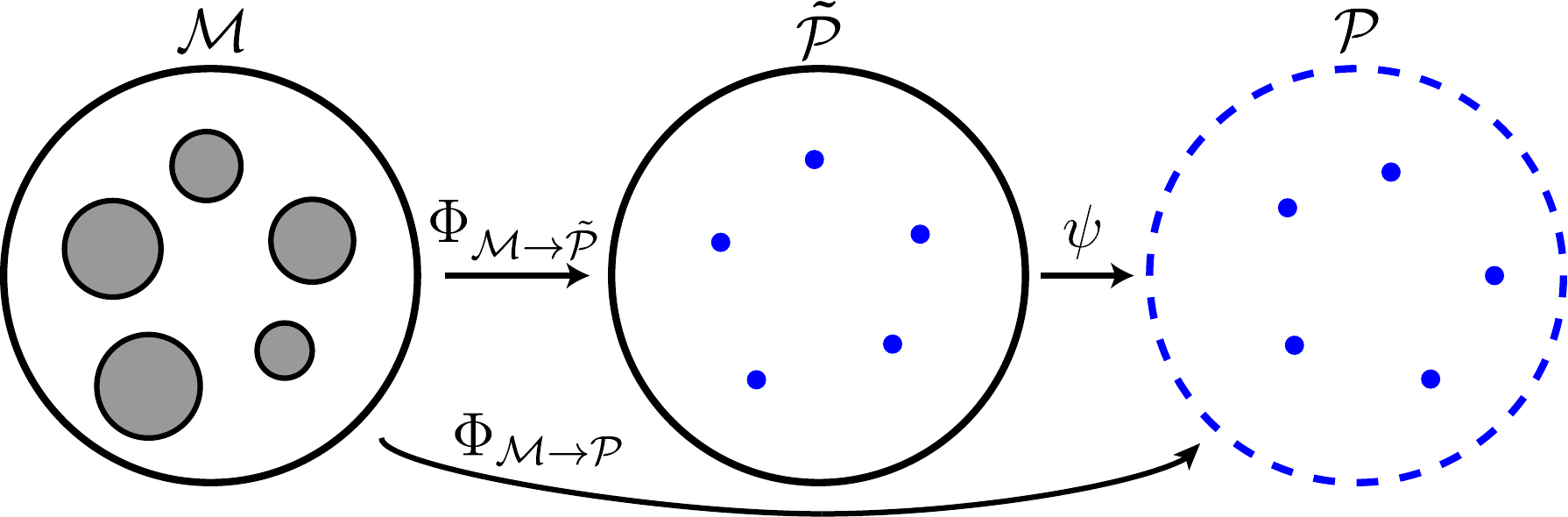}
  \vspace{-0.1in}
  \caption{Illustration of the diffeomorphic transformation from sphere word $\mathcal{M}$ to bounded point world $\Tilde{\mathcal{P}}$
  and to unbounded point world $\mathcal{P}$.}\label{fig:diffeo-trans}
  \vspace{-0.05in}
\end{figure}
\begin{definition} \label{def:harmonic-potential}
The harmonic potential function in a point world, denoted by~$\phi_{\mathcal{P}}:\mathcal{P}\rightarrow \mathbb{R}^+$, is defined as:
\begin{equation}\label{eq:harmonic-point-potential}
\phi_{\mathcal{P}}(x) \triangleq \phi(x,\, P_G) - \frac{1}{K} \sum_{i=1}^M \phi(x, P_i),
\end{equation}
where~$\phi(x,\,q) = \ln\left(\|x-q\|^2\right)$ is
the primitive harmonic function for $x,\,q\in \mathbb{R}^2$;
$\phi(x,\, P_G)$ is the potential for the transformed goal~$P_G=\Phi_{\mathcal{M}\rightarrow \mathcal{P}}(q_G)$,
whereas $\phi(x,\, P_i)$ for the obstacle~$P_i$, where~$i=1,\cdots, M$;
$K\geq 1$ is a tuning parameter. \hfill $\blacksquare$
\end{definition}
Lastly, the logistic function is used to transform the
unbounded range of~$\phi_{\mathcal{P}}$ to a finite interval~$[0,\,\mu]$ for~$\mu\geq 1$.


\subsection{Linear Temporal Logic and B\"uchi Automaton}\label{subsec:ltl}
The basic ingredients of Linear Temporal Logic (LTL) formulas are a set of atomic propositions $AP$, and several Boolean or temporal operators.
Atomic propositions are Boolean variables that can be either true or false.
The syntax of LTL is defined as:
$\varphi \triangleq \top \;|\; p  \;|\; \varphi_1 \wedge \varphi_2  \;|\; \neg \varphi  \;|\; \bigcirc \varphi  \;|\;  \varphi_1 \,\textsf{U}\, \varphi_2$,
where $\top\triangleq \texttt{True}$, $p \in AP$, $\bigcirc$ (\emph{next}),
$\textsf{U}$ (\emph{until}) and $\bot\triangleq \neg \top$.
The derivations of other operators, such as $\Box$ (\emph{always}),
$\Diamond$ (\emph{eventually}), $\Rightarrow$ (\emph{implication})
are omitted here for brevity.
A complete description of the semantics and syntax of LTL can be found in~\cite{baier2008principles}.
Moreover,
there exists a Nondeterministic B\"{u}chi Automaton (NBA) for formula~$\varphi$ as follows:
\begin{definition}\label{def:nba}
A NBA $\mathcal{A}\triangleq (S,\,\Sigma,\,\delta,\,(S_0,\,S_F))$
is a 4-tuple, where~$S$ are the states;
$\Sigma=AP$;
$\delta:S\times \Sigma\rightarrow2^{S}$ are transition relations;
$S_0, S_F\subseteq S$ are initial and {accepting} states. \hfill $\blacksquare$
\end{definition}
An infinite {word} $w$ over the alphabet $2^{AP}$ is defined as an
infinite sequence $W=\sigma_1\sigma_2\cdots, \sigma_i\in 2^{AP}$.
The language of $\varphi$ is defined as the set of words that satisfy $\varphi$,
namely, $\mathcal{L}=Words(\varphi)=\{W\,|\,W\models\varphi\}$ and $\models$ is the satisfaction relation.
Additionally, the resulting \emph{run} of~$w$ within~$\mathcal{A}$
is an infinite sequence~$\rho=s_0s_1s_2\cdots$
such that $s_0\in S_0$, and $s_i\in S$, $s_{i+1}\in\delta(s_i,\,\sigma_i)$ hold for all index~$i\geq 0$.
A run is called \emph{accepting} if it holds that
$\inf(\rho)\cap {S}_F \neq \emptyset$,
where $\inf(\rho)$ is the set of states that appear in $\rho$ infinitely often.
In general, an accepting run has the prefix-suffix structure
from an initial state to an accepting state that is contained in a cyclic path.
Typically, the size of~$\mathcal{A}$ is double exponential to the length of formula~$\varphi$.

\section{Problem Description}\label{sec:problem}
\begin{figure}[t]
  \centering
  \includegraphics[width=0.95\hsize]{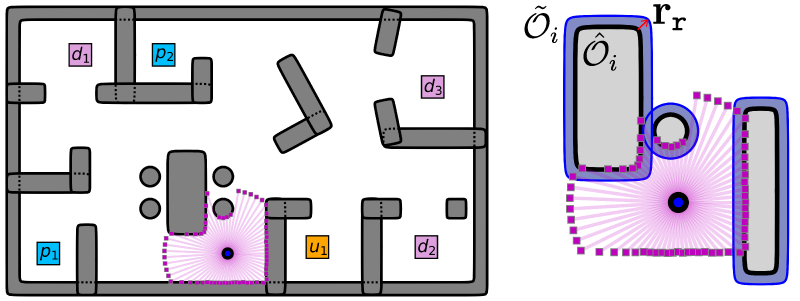}
  \vspace{-0.1in}
  \caption{\textbf{Left:} Robot in the workspace with overlapping squircles
  and several regions of interest.
  \textbf{Right:} Estimated (in black) and inflated (in blue) obstacles.
  }\label{fig:problem}
  \vspace{-0.05in}
\end{figure}
Consider a mobile robot that occupies a circular area with radius~$r_{\texttt{r}}>0$
and follows the unicycle dynamics:
\begin{equation}\label{eq:unicycle}
  \dot{x} = \upsilon \cos(\theta),\; \dot{y} = \upsilon \sin(\theta),\;\dot{\theta} = \omega,
\end{equation}
where~$q=(x,\,y)\in \mathcal{W}$ is the robot position
and $\theta\in [-\pi, \pi]$ as its orientation;
$(\upsilon,\,\omega)$ are its linear and angular velocities as control inputs.
The workspace~$\mathcal{W}_0\subset \mathbb{R}^2$ is compact and connected,
  with~$M$ potentially overlapping internal obstacles~$\mathcal{O}_i \subset \mathcal{W}_0$,
$\forall i \in \{1,\cdots,M\}$.
{Considering the robot size, the outer workspace $\mathcal{W}_0$ and inner obstacle $\mathcal{O}_i$
are inflated by a margin $r_{\texttt{r}}$,
denoted by $\Tilde{\mathcal{W}}_0$ and $\Tilde{\mathcal{O}}_i$.
Therefore, the feasible workspace is given
by~$\mathcal{W} \triangleq \Tilde{\mathcal{W}}_0 \backslash\bigcup_{i=1}^M \Tilde{\mathcal{O}}_i$}.

\begin{figure}[t]
  \centering
  \includegraphics[width=0.95\hsize]{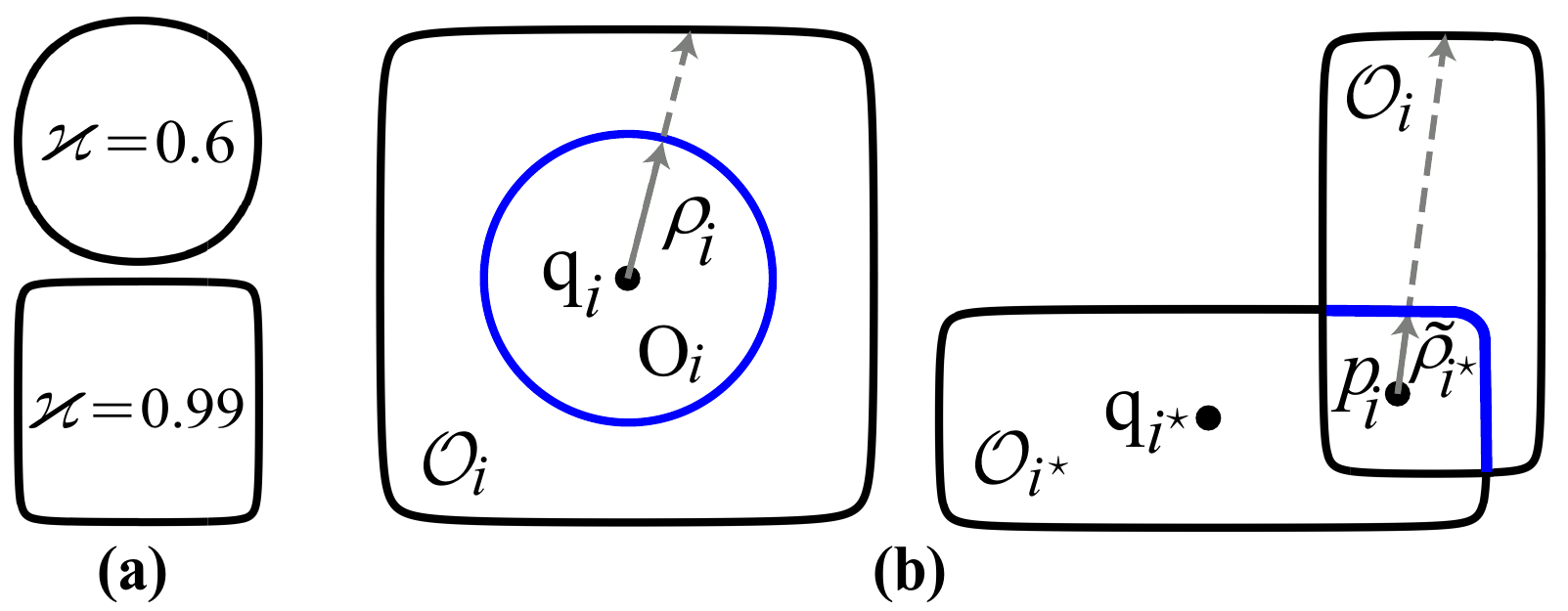}
  \vspace{-0.15in}
  \caption{(a) Squircles with the parameter $\varkappa=0.6$ and $\varkappa=0.99$;
  (b) Ray scaling process. The boundary of the star-shaped obstacle is mapped
  onto the boundary of a sphere (\textbf{Left}). The boundary of the child obstacle
  is mapped onto a segment of the boundary of the parent obstacle (\textbf{Right}).}\label{fig:scaling}
  \vspace{-0.05in}
\end{figure}

Each obstacle~$\mathcal{O}_i$ belongs to a type of obstacle called squircle,
which are particularly useful for representing walls and corners, see~\cite{li2018navigation}.
As shown in Fig.~\ref{fig:scaling}, a squircle interpolates smoothly between a circle and a square,
while avoiding non-differentiable corners.
In particular, a unit squircle centered at the origin in~$\mathbb{R}^2$ is given by:
\begin{equation}\label{eq:squircle}
  \begin{aligned}
    \beta_{\texttt{sc}}(q)\triangleq
    \frac{q^2+\sqrt{q^4-4\varkappa^2\,{[(q^\intercal \, e_1)(q^\intercal \, e_2)]}^2}}{2}-1,
\end{aligned}
\end{equation}
where $\varkappa\in (0,\,1)$ is the curvature;
$e_1,e_2$ are unit basis in~$\mathbb{R}^2$.
Non-unit and rotated squircles with general centers can be derived via
scaling, translation and rotation.

\begin{assumption}\label{assump:squircle}
  The workspace boundary and inner obstacles all
    follow the model of squircles in~\eqref{eq:squircle}.
  \hfill $\blacksquare$
\end{assumption}

Initially at~$t=0$, the workspace is only \emph{partially} known to the robot,
i.e., the outer boundary and some inner obstacles.
Starting from any valid initial state~$(q_0,\,\theta_0)$,
the robot can navigate within the workspace and observe more obstacles,
via a range-limited sensor modeled as follows:
\begin{equation}\label{eq:sense}
  \mathcal{S}(q) \triangleq \left\{\hat{q}\in \mathcal{W}
  \,|\, \left(\hat{q}\in \mathcal{D}_{r_{\texttt{s}}}(q)\right)
  \wedge \left(\mathcal{L}(q,\,\hat{q})\subset \mathcal{W}\right)\right\},
\end{equation}
where~$\mathcal{S}(q)$ is the set of points~$\hat{q}$
observed by the robot at position~$q\in \mathcal{W}$;
$\mathcal{D}_{r_{\texttt{s}}}(q)$ is a disk centered at~$q$ with
radius~$r_{\texttt{s}}$ and $\mathcal{L}(q,\,\hat{q})$ is the line
connecting~$q$ and~$\hat{q}$.
{As shown in Fig.~\ref{fig:problem}, it returns the 2D point cloud
from the robot to any blocking surface within the sensing range.}
This model mimics a $360^\circ$ Lidar scanner as also used
in~\cite{rousseas2022trajectory}.
Given the measurements, an observed obstacle
can be estimated accordingly.

\begin{assumption}\label{assump:accurate}
  An obstacle~$\mathcal{O}_i\in \mathcal{W}_0$ can be
    estimated accurately once~$S(q)$ in~\eqref{eq:sense}
  intersects with its occupied area.
  \hfill $\blacksquare$
\end{assumption}

Note that the implication and relaxation of
Assumption~\ref{assump:accurate}
are discussed in Sec.~\ref{subsubsec:init-nf}
and~\ref{subsubsec:online-squircle} in the sequel.
Lastly, there is a set of non-overlapping regions of
interest~$g_n\subset \mathcal{W}$, $n=1,\cdots,N$ within the freespace.
With slight abuse of notation, the associated atomic propositions are also
denoted by~$G=\{g_n\}$, standing
for ``the robot is within region $g_n$, i.e., $q\in g_n$''.
The desired task is specified as a LTL formula~$\varphi$ over~$G$,
i.e., $\varphi=LTL(G)$ by the syntax described in Sec.~\ref{subsec:ltl}.
Given the robot trajectory~$\mathbf{q}$, its trace is given by the sequence of
regions over time, i.e.,
$\omega(\mathbf{q})=g_{\ell_1}g_{\ell_2}\cdots$,
where~$g_{\ell_k}\in G$ and~$q(t_k)\in g_{\ell_k}$,
for some time instants~$0\leq t_k\leq t_{k+1}$ and $k\in \mathbb{Z}$.

Thus, the objective is to design an online control and planning strategy
for system~\eqref{eq:unicycle} such that
starting from an initial pose~$(q_0,\,\theta_0)$,
the trace of the resulting trajectory $\omega(\mathbf{q})$
fulfills the given task~$\varphi$,
while avoiding collision with all obstacles.



\section{Proposed Solution}\label{sec:solution}
As illustrated in Fig.~\ref{fig:diagram},
the proposed solution is a hybrid control framework
as two-layer harmonic trees (HT) that combines harmonic potentials for navigation
and oriented search trees for planning.
Initially, the automaton-guided search trees and the orientation-aware harmonic potentials
are constructed in a dependent manner given the partially-known workspace.
Then, as the robot explores more obstacles,
an online adaptation scheme is proposed to revise the search tree and
update the harmonic potentials recrusively and simultaneously.

\subsection{Initial Synthesis}\label{subsec:initialization}

\subsubsection{Two-layer and Automaton-guided Harmonic Trees}
\label{subsubsec:init-tree}
As described in Sec.~\ref{subsec:ltl},
the NBA associated with~$\varphi$ is given by~$\mathcal{A}_{\varphi}=(S,\,\Sigma,\,\delta,\,(S_0,\,S_F))$,
which captures all potential traces that satisfy the task.
To begin with, the initial \emph{navigation map} is constructed as a weighted and fully-connected
graph~$\mathcal{G}\triangleq (\widehat{G},\,E,\,d,\,(g_0,\theta_0))$,
where: (i) $\widehat{G}=G\times \Theta$ is the set of regions of interest
plus a set of orientations $\Theta\subset [0, 2\pi)$;
(ii) $E\subset  \widehat{G} \times \widehat{G}$ is the set of transitions;
(iii) $d: E \rightarrow \mathbb{R}^+$ is the cost function,
which is initialized as~$d((g,\theta),\,(g',\theta'))\triangleq \|g-g'\|_2+w|\theta-\theta'|$,
$\forall (g,\theta),\,(g',\theta')\in \widehat{G}$ and parameter~$w>0$;
and $(g_0,\theta_0)$ is the initial pose.
Note that $E$ is initialized as fully-connected since the actual feasibility
and cost can only be determined after the associated controllers are constructed.

Given~$\mathcal{G}$ and~$\mathcal{A}_{\varphi}$, the standard model-checking procedure
is followed to find the task plan.
Namely, their synchronized product is built
as~$\widehat{\mathcal{A}}\triangleq \mathcal{G}\times \mathcal{A}_{\varphi}=
(\widehat{S},\,\widehat{\delta},\,\widehat{d},\,(\widehat{S}_0,\widehat{S}_F))$,
where~$\widehat{S}=\widehat{G}\times S$;
$\widehat{S}_0,\widehat{S}_F\subset \widehat{S}$ are the sets of initial and accepting states;
$\widehat{\delta} \subset \widehat{S}\times \widehat{S}$
{that}~$(\langle g,s\rangle, \langle g', s'\rangle)\in \widehat{\delta}$
if~$(g,g')\in E$ and~$s'\in \delta(s,\{g\})$;
$\widehat{d}(\langle g,s\rangle, \langle g',s'\rangle)=d(g,g')$,
$\forall (\langle g,s\rangle, \langle g',s'\rangle)\in \widehat{\delta}$.
Note {that} the product~$\mathcal{\widehat{A}}$ is still a B\"uchi automaton,
of which the accepting run satisfies the prefix-suffix structure.
Namely, consider the following run of~$\widehat{\mathcal{A}}$:
$\widehat{\mathbf{S}}=\widehat{s}_1\widehat{s}_2\cdots \widehat{s}_L
  (\widehat{s}_{L+1}\widehat{s}_{L+2}\cdots$
$\widehat{s}_{L+H})^{\omega}$,
which an infinite sequence of~$\widehat{S}$
with~$\widehat{s}_1 \cdots \widehat{s}_L$ being the prefix
and~$\widehat{s}_{L+1}\cdots \widehat{s}_{L+H}$ being the suffix repeated infinitely often;
$\widehat{s}_1\in \widehat{S}_0$ and~$\widehat{s}_{L+1}\in \widehat{S}_F$; and~$L,H\geq 1$.
A nested Dijkstra algorithm is used to find the best pair of initial
and accepting states~$(\widehat{s}^\star_0, \widehat{s}^\star_F)$.
Thus, the optimal plan given the initial environment is obtained by
projecting~$\widehat{\mathbf{S}}$ onto~$\widehat{G}$, i.e.,
\begin{equation}\label{eq:plan}
  \widehat{\mathbf{g}}=\widehat{g}_1\widehat{g}_2\cdots \widehat{g}_L
  (\widehat{g}_{L+1}\widehat{g}_{L+2}\cdots \widehat{g}_{L+H})^{\omega},
\end{equation}
where~$\widehat{g}_{\ell}=\widehat{s}_\ell|_{\widehat{G}}$ are the regions.
More algorithmic details can be found in~\cite{guo2015multi}.

To avoid oscillation or jitters and long detours as mentioned in Sec.~\ref{subsec:intro-related},
the structure of oriented harmonic trees (HT) is proposed.
More specifically, the oriented HT associated with $(\widehat{g}_\ell,\,\widehat{g}_{\ell+1})$
is a tree structure defined by a 4-tuple:
\begin{equation}\label{eq:tree}
\mathcal{T}_{\widehat{g}_\ell \rightarrow \widehat{g}_{\ell+1}}\triangleq \big(V,\,B,\,\gamma,\,(\nu_0,\,\nu_G)\big),
\end{equation}
where~$V\subset \mathcal{W} \times (-\pi,\,\pi]$ is the set of vertices;
$B \subset V\times V$ is the set of edges;
$\gamma: B \rightarrow \mathbb{R}_{\geq 0}$ returns the edge cost to be estimated;
$\nu_0=\widehat{g}_\ell$ and $\nu_G=\widehat{g}_{\ell+1}$ are the initial and target poses.
The goal is to find a sequence of vertices in~$\mathcal{T}$
as the path from~$\nu_0$ to~$\nu_G$.
Initially, $V=\{\nu_0,\,\nu_G\}$ and $B=\emptyset$.
Then, the set of vertices can be generated in various ways, e.g.,
the visibility graph from~\cite{huang2004dynamic} or sampling-based methods from~\cite{lavalle2006planning}.
It is worth mentioning that these vertices should be augmented
by \emph{orientations} if not already.
Afterwards, any vertex is connected to all vertices within its \emph{free} vicinity
with an estimated cost.
Thus, given the weighted and directed tree~$\mathcal{T}_{\widehat{g}_\ell \rightarrow \widehat{g}_{\ell+1}}$,
the shortest path from~$\nu_0$ to~$\nu_G$
is determined by A$^\star$ from~\cite{lavalle2006planning}, denoted by:
\begin{equation}\label{eq:path}
  \mathbf{P}_{\widehat{g}_\ell \rightarrow \widehat{g}_{\ell+1}} \triangleq \nu_0\nu_1 \cdots \nu_{N-1}\nu_G,
\end{equation}
where~$\nu_n \in V$ and~$(\nu_n,\, \nu_{n+1})\in B$,
$\forall n\in [0,\,N-1]$.
In other words, each vertex along the path serves as the intermediate waypoints
to navigate from region~$\widehat{g}_\ell$ to $\widehat{g}_{\ell+1}$.

\subsubsection{Orientation-aware Harmonic Potentials}
\begin{figure}[t]
  \centering
  \includegraphics[width=0.98\hsize]{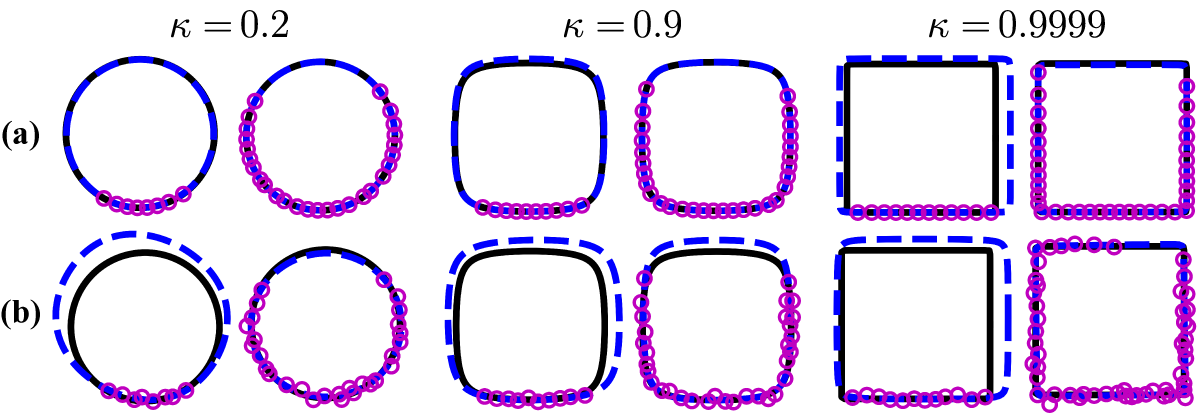}
  \vspace{-0.1in}
  \caption{Estimated squircles (in blue dashed lines)
      under different curvatures~$\kappa$ and
      different distributions of (a) accurate or (b)
      noisy measurements.}\label{fig:squircle_esti}
  \vspace{-0.05in}
\end{figure}
\label{subsubsec:init-nf}
As explained in~\eqref{eq:sense},
a 2D point cloud is returned that consists of points on any obstacle surface
within the sensing range.
More specifically, denoted by~$\mathbf{D}_t=\{d_j\}$ the
set of 2D points that are already
transformed from the local coordinate to global coordinate
where~$d_j\in \partial \mathcal{W}$.
To begin with, these data points are divided into~$K$ clusters
representing~$K$ separate obstacles, i.e.,
$\mathbf{D}_{t} = \{\mathbf{D}_{t,k}\}$,
by e.g., checking the relative distance
and change of curvature between consecutive points.
The exact thresholds would depend on the specification of the Lidar scanner.
Then, each cluster is fitted to the model of squircles in~\eqref{eq:squircle}
to estimate the curvature, translation, scaling and rotation,
via general nonlinear optimization solvers, e.g.,~\cite{gavin2019levenberg}.
As shown in Fig.~\ref{fig:squircle_esti}, the estimation accuracy relies heavily
on the actual parameters of the model in~\eqref{eq:squircle},
the noise level and the distribution of the measurements.
More specifically, the estimation is rather
accurate when the curvature~$\varkappa$
is close to~$0$ such as circles;
however becomes inaccurate or uncertain when when~$\varkappa$
is close to~$1$ such as rectangles and when the data
points are few or noisy.
In the later case, either a pre-stored database of common
obstacles and their shapes can be retrieved based on
object recognition
as adopted in~\cite{vasilopoulos2022reactive},
or an optimistic strategy that chooses the one
with the minimum area among the candidates.
For the rest of this section, it is assumed that the estimation
is accurate by Assumption~\ref{assump:accurate}.
Relaxation of this assumption via online adjustment
as more measurements are gathered
is discussed in Sec.~\ref{subsubsec:online-squircle}.

Consequently, denote by~$\{\hat{\mathcal{O}}_t\}$ the collection of obstacles
detected and fitted at time~$t\geq 0$.
An example is shown in Fig.~\ref{fig:problem}, where an initial workspace model
is constructed given the sensory data at~$t=0$.
Given the set of squircle obstacles~$\{\hat{\mathcal{O}}_0\}$,
the initial navigation function~$\varphi_{\texttt{NF}}(q)$ is constructed by \emph{three} major diffeomorphic transformations:
(i) $\Phi_{\mathcal{F}\rightarrow \mathcal{S}}$ transforms the forest of stars into a star world via ``{purging}'';
(ii) $\Phi_{\mathcal{S}\rightarrow \mathcal{M}}$ transforms the star world into its model sphere world;
and (iii) $\Phi_{\mathcal{M}\rightarrow \mathcal{P}}$ transforms the sphere world into a point world.
The complete navigation function for the original workspace is given by:
\begin{equation}\label{eq:complete-nf}
  \varphi_\texttt{NF}(q) = \sigma \circ \phi_{\mathcal{P}}
  \circ \Phi_{\mathcal{M}\rightarrow \mathcal{P}}
  \circ \Phi_{\mathcal{S}\rightarrow \mathcal{M}}
  \circ \Phi_{\mathcal{F}\rightarrow \mathcal{S}}(q),
\end{equation}
where~$\sigma \circ \phi_{\mathcal{P}}$ is defined in~\eqref{eq:harmonic-point-potential};
and the transformation~$\Phi_{\mathcal{M}\rightarrow \mathcal{P}}$ is given in~\eqref{eq:exact-tf-m-p}.
The remaining part of this section describes
the two essential and nontrivial
transformations~$\Phi_{\mathcal{S}\rightarrow \mathcal{M}}$ and~$\Phi_{\mathcal{F}\rightarrow \mathcal{S}}$.

\textbf{Star-to-Sphere Transformation}.
The star world~$\mathcal{S}$ has an outer boundary of
squircle workspace~$\mathcal{O}_0=\{q \in \mathbb{R}^2|\beta_0(q)\leq 0\}$
and~$M$ inner squircle obstacles~$\mathcal{O}_i=\{q \in \mathbb{R}^2|\beta_i(q)\leq 0\}$, where~$\beta_i(q)$ is the obstacle function defined in~\eqref{eq:squircle}, for~$i=0,1,\cdots, M$.
The star-to-sphere transformation is constructed by the ray scaling process from~\cite{rimon1990exact},
as shown in Fig.~\ref{fig:scaling}.
To begin with, define the following scaling factor:
\begin{equation}\label{eq:scaling}
v_0(q) \triangleq \rho_0\frac{1-\beta_0(q)}{||q-\mathbf{c}_0||},\quad
v_i(q) \triangleq \rho_i\frac{1+\beta_i(q)}{||q-\mathbf{c}_i||},
\end{equation}
where~$\mathbf{c}_i$ is the geometric center of the associated squircle
and~$\rho_i$ is the radius of the transformed sphere.
Moreover, the translated scaling map~$T_i$ for each obstacle~$\mathcal{O}_i$ is defined by:
\begin{equation}\label{eq:translated-scaling-map}
  T_i(q)\triangleq v_i(q)\,(q-\mathbf{c}_i)+\mathbf{c}_i,
\end{equation}
for~$i=0,1,\cdots,M$.
Consequently, the transformation from star world~$\mathcal{S}$ to its model sphere world~$\mathcal{M}$ is given by:
\begin{equation}\label{eq:tf-s-m}
  \Phi_{\mathcal{S}\rightarrow \mathcal{M}} \triangleq  \Big(1-\sum_{i=0}^M \sigma_i(q)\Big)\,{\rm id}(q)+\sum_{i=0}^M \sigma_i(q)\, T_i(q),
\end{equation}
where~$\sigma_i(q)\triangleq \frac{\gamma_G(q)\overline{\beta}_i(q)}{\gamma_G(q)\overline{\beta}_i(q)+ \lambda\beta_i(q)}$
is the analytic switch for the workspace boundary and obstacles;
${\rm id}(q)$ is the identity function;
$\lambda>0$ is a parameter;
$\gamma_G(q) \triangleq \|q-q_G\|^2$ is the distance-to-goal;
and~$\overline{\beta}_i(q)\triangleq \prod_{j=0,j\ne i}^M \beta_j(q)$
is the omitted product.

\textbf{Leaf-Purging Transformation}. As described in~\cite{loizou2022mobile, li2018navigation, rimon1990exact},
besides disjointed star worlds, it is essential to consider the workspace
formed by unions of \emph{overlapping} stars, called the forest of stars.
It consists of several disjointed clusters of obstacles as trees of stars,
which in turn is a finite union of overlapping star obstacles whose
adjacency graph is a tree.
Since overlapping stars can not be transformed directly as a whole obstacle,
each tree has to be transformed via successively \emph{purging} its leaves.
Specifically, a forest of stars is described
by {$\mathcal{F}=\mathcal{W}_0 \backslash \bigcup_{n=1}^N\mathcal{T}_n$},
where~$\mathcal{T}_n$ is the~$n$-th tree of stars among the~$N$ trees.
Each tree of stars has a unique root, and its
obstacles are arranged in a parent-child relationship.
An example is shown in Fig.~\ref{fig:forest},
where the trees have different depths according to the level of leaves in the tree.
Without loss of generality,
denote by~$d_n$ the depth of the tree~$\mathcal{T}_n$
and~$\mathcal{L}$ the set of indices for all leaf obstacles in all trees,
and by~$\mathcal{I}$ the set of indices associated with all obstacles within~$\mathcal{F}$.
Furthermore, denote by~$\mathcal{O}_i\triangleq\{q\in\mathbb{R}^2: \beta_i(q)\leq 0\}$ the leaf obstacles,
where~$\beta_i(q)$ is the obstacle function defined in~\eqref{eq:squircle}.
The \emph{parent} of obstacle~$\mathcal{O}_i$ is denoted by
$\mathcal{O}_{i^{\star}}$,
of which the centers are chosen to be the same
and denoted by~$p_i \in \mathcal{O}_i\bigcap \mathcal{O}_{i^{\star}}$.
Then, the diffeomorphic transformation from the forest of stars $\mathcal{F}$ to the star world~$S$
is constructed via the successive purging transformations for each tree as follows:
\begin{equation}\label{eq:tf-f-s}
  \Phi_{\mathcal{F}\rightarrow \mathcal{S}}(q) \triangleq \Phi_{N} \circ \cdots  \Phi_{2}\circ \Phi_{1}(q),
\end{equation}
where~$\Phi_{n}(q)$ is the purging transformation for the~$n$-th tree of stars,
where~$n=1,\cdots,N$.
It has the following format:
\begin{equation}\label{eq:purging-tf}
  \Phi_{n}(q) \triangleq f_{n,1} \circ f_{n,2} \cdots \circ f_{n,d_n}(q),
\end{equation}
where~$f_{n,i}(q)$ is the purging transformation for the~$i$-th leaf obstacle in the~$n$-th tree for~$i=1,\cdots,d_n$,
which is constructed by the ray-scaling process to purge the leaf
obstacle into its parent as shown in Fig~\ref{fig:scaling}.
Due to the specific representation of the squircle, the length of rays from the common center~$p_i$ to the boundary of the parent obstacle~$\mathcal{O}_{i^{\star}}$, is calculated by:
\begin{equation}\label{eq:sc_length_ray}
\Tilde{\rho}_{i^{\star}}(\hat{q}) \triangleq
    \frac{1}{\|\Tilde{\mathbf{A}}_{i^{\star}}^{-1}\hat{q}\|}
    \, \rho_{sc} \left(\frac{\Tilde{\mathbf{A}}_{i^{\star}}^{-1}\hat{q}}{\|\Tilde{\mathbf{A}}_{i^{\star}}^{-1}\hat{q}\|}\right),
\end{equation}
where $\hat{q}=\frac{q-p_i}{\|q-p_i\|}$ is the normalized vector of $q - p_i$;
$\rho_{sc}(\cdot)$ is the length of the rays for unit squircle drived from~\eqref{eq:squircle};
and $\Tilde{\mathbf{A}}_{i^{\star}}\in \mathbb{R}^{2 \times 2}$ is a piecewise scaling matrix
tailored for the parent squircle which is diagonal with the following entries:
\begin{equation}\label{eq: scaling-matrix-entries}
  \Tilde{a}_{i^{\star}}^{\ell}(\hat{q}) \triangleq
    a_{i^{\star}}^{\ell} + \mathbf{sgn}(\hat{q}^\intercal e_{\ell}) [\mathbf{c}_{i^{\star}} - p_i]^\intercal e_{\ell},
\end{equation}
for~$\ell=1,2$, where~$a_{i^{\star}}^{\ell}$ are the entries of the diagonal scaling matrix~$\mathbf{A}_{i^{\star}}$ of the parent squircle;
$\mathbf{c}_{i^{\star}}$ is the geometric center of the parent obstacle;
and~$e_{\ell}$ are two base vectors for~$\ell=1,2$.
\begin{figure}[t]
  \centering
  \includegraphics[width=0.9\hsize]{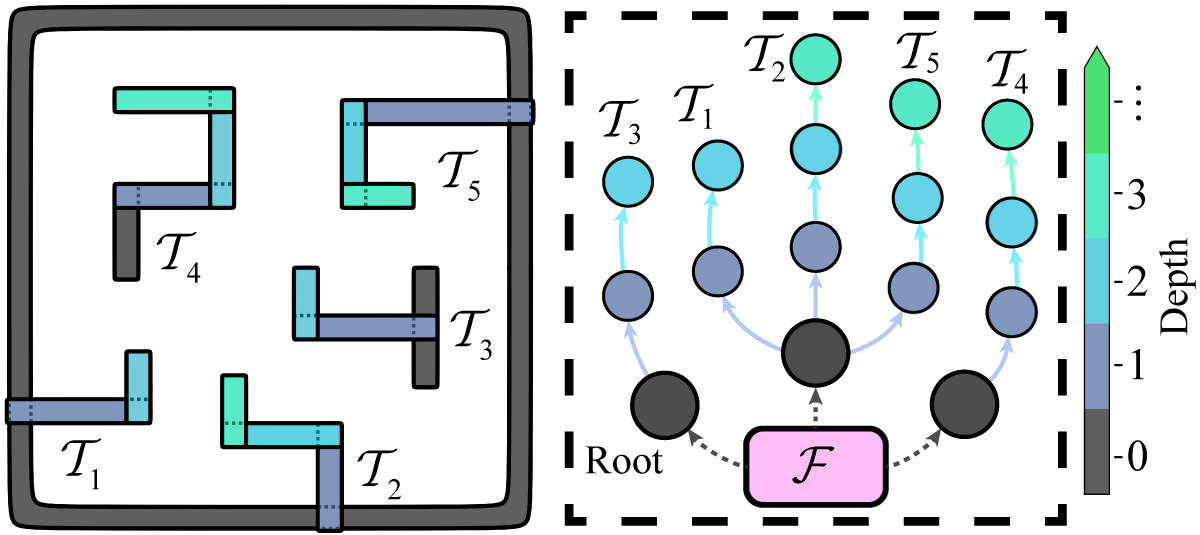}
  \vspace{-0.1in}
  \caption{{A forest world~$\mathcal{F}$ with overlapping squircles (\textbf{Left}),
    which consists of~$5$ trees of stars~$\mathcal{T}_i$ with different depths
    from the root to the leaves (\textbf{Right}).}}
  \label{fig:forest}
  \vspace{-0.05in}
\end{figure}

\begin{lemma}\label{lemma:smooth_sc_length_ray}
  The length of rays~$\Tilde{\rho}_{i^{\star}}(q)$ is smooth in {$\mathcal{W}_0 \backslash \mathcal{O}_i\bigcup \mathcal{O}_{i^{\star}}$}.
\end{lemma}
\begin{proof}
The derivative of~$\Tilde{\rho}_{i^{\star}}(q)$ is given by the chain rule:
$$
\nabla\Tilde{\rho}_{i^{\star}}(q)=\frac{\|q-p_i\|^2\mathbf{I}-(q-p_i)(q-p_i)^\intercal}{\|q-p_i\|^3}\nabla\Tilde{\rho}_{i^{\star}}(\hat{q}),
$$
where the derivative of~$\Tilde{\rho}_{i^{\star}}(\hat{q})$
boils down to computing the derivatives of~$\rho_{sc}(q)$ and~$Z(\hat{q})=\Tilde{\mathbf{A}}_{i^{\star}}^{-1}\hat{q}$ in~\eqref{eq:sc_length_ray}.
The smoothness of~$\rho_{sc}(q)$ and its derivative can be obtained directly from~\eqref{eq:squircle}.
Since $\nabla a_{i^{\star}}^{\ell}(\hat{q})=\nabla(a_{i^{\star}}^{\ell} + \mathbf{sign}(\hat{q}^\intercal e_{\ell}) [q_{i^{\star}} - p_i]^\intercal e_{\ell})=0$, for $\ell=1,2$,
the derivative of~$Z(\hat{q})$ is calculated and simplified as:
$$
\nabla Z(\hat{q})=\Tilde{\mathbf{A}}_{i^{\star}}^{-1}\mathbf{I} -
\Tilde{\mathbf{A}}_{i^{\star}}^{-1}[\nabla\Tilde{\mathbf{A}}_{i^{\star}}]\Tilde{\mathbf{A}}_{i^{\star}}^{-1}\hat{q}
= \Tilde{\mathbf{A}}_{i^{\star}}^{-1}.
$$
Besides, as~$\|q-p_i\|>0$ holds for every
point {$q \in \mathcal{W}_0 \backslash \mathcal{O}_i\bigcup \mathcal{O}_{i^{\star}}$},
the derivative of~$\Tilde{\rho}_{i^{\star}}(q)$ exists and is well-defined.
On the other hand, $\nabla\Tilde{\rho}_{i^{\star}}(q)$ is continuous since its compositions are all continuous.
Therefore,~$\Tilde{\rho}_{i^{\star}}(q)$ is at least~$C^1$ smooth.
Additionally, higher-order smoothness can be proven by the same procedure, which is omitted here.
\end{proof}

Given the length of rays~$\Tilde{\rho}_{i^{\star}}(q)$,
the star deforming factor for each leaf obstacle~$\mathcal{O}_i$ is calculated by:
\begin{equation}\label{eq:deforming-factor}
  v_i(q) = \Tilde{\rho}_{i^{\star}}(q)\,\frac{1+\beta_i(q)\Tilde{\kappa}_i(q)}{\|q-p_i\|},
\end{equation}
{where~$\Tilde{\kappa}_i(q)$ is defined by:
$
\Tilde{\kappa}_i(q) \triangleq
\beta_{i^{\star}}(q)+(\beta_i(q)-2E_i)+\sqrt{\beta^2_{i^{\star}}(q)+((\beta_i(q)-2E_i))^2}.
$
The parameter $E_i > 0$ is a geometric constant satisfying that
$\mathcal{O}_i(2E_i) \bigcap \mathcal{O}_j(2E_j)=\emptyset$
and $\gamma_G^{-1}([0, 2E_G]) \bigcap \mathcal{O}_i(2E_i)=\emptyset$
with $E_G > 0$ being a geometric constant
and $\mathcal{O}_i(x)\triangleq \{q \in \mathbb{R}^2|\beta_i(q)\leq x\}$ with~$x>0$,
for $i,\,j \in \mathcal{I}$ with $i \neq j\,$ and $i,j \neq j^{\star}$,
where $\mathcal{O}_{i^\star}$ and $\mathcal{O}_{j^\star}$ are the parent obstacles
of $\mathcal{O}_{i}$ and $\mathcal{O}_{j}$ in the tree, respectively.
Furthermore,
the translated scaling map~$T_i$ for each obstacle~$\mathcal{O}_i$ with the common center~$p_i$ is defined by:
\begin{equation}\label{eq:purging-scaling}
  T_i(q)\triangleq v_i(q)\,(q-p_i)+p_i.
\end{equation}
Therefore, the purging transformation~$f_{n,i}(q)$ for the~$i$-th leaf obstacle~$\mathcal{O}_i$ in the~$n$-th tree is defined by:
\begin{equation}\label{eq:purging}
  f_{n,i}(q) \triangleq \big(1-\sigma_i(q,\xi_i)\big)\,{\rm id}(q) + \sigma_i(q, \xi_i)\, T_i(q),
\end{equation}
where~$\sigma_i(q,\xi_i)\triangleq \frac{\gamma_G(q)\overline{\beta}_i(q)}{\gamma_G(q)\overline{\beta}_i(q)+ \xi_i\beta_i(q)}$ is the analytic switch;
$\xi_i>0$ is a positive parameter;
$\gamma_G(q)$ is the same as in~\eqref{eq:tf-s-m};
and~$\overline{\beta}_i(q)$ is the omitted product defined by:
$$
\overline{\beta}_i(q) \triangleq \Big(\prod_{j \in \mathcal{I}\backslash \{i,i^{\star}\}}\beta_j(q)\Big)
\Big(\prod_{j \in \mathcal{L}\backslash \{i\} } \beta_j(q)\Big)\Tilde{\beta}_i(q),
$$
where~$\Tilde{\beta}_i(q)$ is defined by:
$
\Tilde{\beta}_i(q) \triangleq
\beta_{i^{\star}}(q)+(2E_i-\beta_i(q))+\sqrt{\beta^2_{i^{\star}}(q)+(2E_i-\beta_i(q))^2}
$.
The positive geometric constant~$E_i$ satisfies the same condition as previously described.

\begin{remark}\label{remark:compare-purging}
The original purging method proposed in~\cite{rimon1990exact} deals with
overlapping obstacles via Boolean combinations
and applies only to planar and parabolic obstacles.
The work in~\cite{li2018navigation} proposes a simpler method
for computing the ray scaling transformations using varying ray lengths.
However, these approaches purge all the leaves \emph{at once},
which is not applicable to the dynamic scenarios where the obstacles are added to the workspace one by one.
Thus, a novel method is proposed in~\eqref{eq:purging} that purges the leafs \emph{one by one}.
\hfill $\blacksquare$
\end{remark}
\begin{lemma}\label{lemma:point-world-nf}
  { Given a workspace $\mathcal{W}$ containing $N$ tree-of-stars with different depth $d_n\geq 0$,
  for $n=1, \cdots N$,
  the complete harmonic potential function~$\varphi_\texttt{NF}(q)$ in~\eqref{eq:complete-nf} is a
  valid navigation function for the workspace~$\mathcal{W}$,
  if the parameters~$K > N$ and $\mu\geq 1$ hold in~\eqref{eq:harmonic-point-potential};
  $\lambda > \Lambda$ for a positive number $\Lambda >0$;
  and $\xi_i > \Xi_i$ for some positive numbers $\Xi_i > 0$, for $i = 0, \cdots, d_n$.}
\end{lemma}
\begin{proof}
Similar statements have been proven
in~\cite{rousseas2021harmonic, filippidis2011adjustable, loizou2022mobile},
thus detailed proofs are omitted here and refer the readers to e.g.,
Theorem~1 of~\cite{loizou2022mobile}, Theorem 2 in~\cite{li2018navigation} and
Proposition~5 of~\cite{loizou2017navigation}.
{Briefly speaking, it is shown that the transformation $\Phi_{\mathcal{S}\rightarrow \mathcal{M}}$
$\Phi_{\mathcal{F}\rightarrow \mathcal{S}}$ are diffeomorphic as long as $\lambda > \Lambda$
and $\xi_i > \Xi_i$ for some sufficient large numbers $\Lambda$ and $\Xi_i$,
and the final potential function $\varphi_{\texttt{NF}}(q)$ has
a unique global minimum at the goal~$q_{\texttt{G}}$ and more importantly,
a set of isolated saddle points of measure zero if~$K > N$ and $\mu\geq 1$ hold.}
Since the purging process proposed in~\eqref{eq:purging} is novel,
this proof mainly shows that~$f_{n,i}(q)$ still satisfies the diffeomorphic conditions:
(i) the Jacobian of~$f_{n,i}(q)$ is nonsingular;
(ii) $f_{n,i}(q)$ is a bijection on the boundary.
To begin with, since~$\Tilde{\rho}_{i^{\star}}(q)$ is proven to be smooth in Lemma~\ref{lemma:smooth_sc_length_ray},
the Jacobian of~$f_{n,i}$ exists and is given by:
$$
\begin{aligned}
J_{f_{n,i}} = &\,\sigma_i(q-p_i)\nabla v_i^{\intercal}
+ (v_i - 1)(q-p_i) \nabla \sigma_i^{\intercal}\\
&+ (1-\sigma_i)\mathbf{I} + \sigma_i v_i \mathbf{I}.
\end{aligned}
$$
Similar to Lemmas~7 and~8 in~\cite{li2018navigation},
it can be shown that there exists a positive constant~$\Xi_i$,
such that if~$\xi_i > \Xi_i$,~$J_{f_{n,i}}$ is nonsingular within the set near~$\mathcal{O}_i$,
denoted by $\mathcal{O}_i(\epsilon_i)=\{q \in \mathbb{R}^2|\beta_i \leq \epsilon_i\}$,
and the set away from~$\mathcal{O}_i$,
denoted by $\mathcal{A}(\epsilon_i)=\{q \in \mathbb{R}^2|\beta_i \geq \epsilon_i\}$, for any~$\epsilon_i>0$.
In addition, it remains to show the injectivity and surjectivity of~$f_{n,i}(q)$.
On the boundary of~$\mathcal{O}_i$,
it holds that~$f_{n,i}(q)|_{q \in \partial\mathcal{O}_i}=\frac{\Tilde{\rho}_{i^{\star}}(q)}{\|q-p_i\|}(q-p_i) + p_i$.
Assume that there exists two points~$q,\,q^{\prime}\in\partial\mathcal{O}_i$ such that
$f_{n,i}(q)=f_{n,i}(q')$, it can be derived that
$q-p_i$ and~$q^{\prime}-p_i$ are on the same ray,
yielding to the contradiction that~$q,\,q^{\prime}$ are the same point.
Thus,~$f_{n,i}(q)$ is injective on~$\partial \mathcal{O}_i$.
On the other hand, since it has been shown that
the Jacobian of~$f_{n,i}(q)$ is nonsingular if~$\xi_i > \Xi_i$,
$f_{n,i}(q)$ is a local homeomorphism according to the Inverse Function Theorem.
Since a local homeomorphism from a compact space into a connected one is surjective,
it follows that~$f_{n,i}(q)$ is a bijection on the boundary.
\end{proof}

The derived potential fields in~\eqref{eq:complete-nf} can be used to drive
holonomic robots from any initial point to the given goal point~$q_G$,
e.g., by following the negated
gradient in~\cite{rousseas2021harmonic, rousseas2022trajectory, filippidis2011adjustable, loizou2022mobile}.
However, the final orientation of the robot at the destination \emph{can not} be controlled,
yielding it impractical for the oriented path~$\mathbf{P}_{\widehat{g}_\ell \rightarrow \widehat{g}_{\ell+1}}$ in~\eqref{eq:path}.
Thus, an additional two-step rotation is proposed for the potential fields above.

\textbf{Two-step Rotation to Harmonic Potentials}:
As shown in Fig.~\ref{fig:vector-traj},
the main idea is to design a two-step rotation transformation to the harmonic potentials,
such that a ``dipole-like'' field is formed near the goal point with the desired orientation.
More specifically, the transformation is given by:
\begin{equation}\label{eq:two-step-transformation}
  \Gamma(q) \triangleq \mathbf{R}\big(\theta_2(q)\big)\, \mathbf{R}\big(\theta_1(q)\big),
\end{equation}
where~$\mathbf{R}(\theta)$ is the standard rotation matrix, i.e.,
$\mathbf{R}(\theta)=[\cos\theta$,
$-\sin\theta; \sin\theta, \cos\theta]$ for~$\theta \in [0, -\pi)$;
~$\theta_1(q)$ is the angle to be rotated
such that the direction of the transformed potential is aligned
with the direction from~$q_G$ to~$q$ near the goal, i.e.,
\begin{equation}\label{eq:theta1}
\theta_1(q) \triangleq s_d(q)\delta_\theta(q)  +
\mathbf{sgn}\big(\delta_\theta(q)\big)[1-s_d(q)]\delta_c,
\end{equation}
where~$\delta_\theta(q) \triangleq \theta_{q-q_G}-\theta_{\nabla\varphi_\texttt{NF}(q)}$ is
the relative angle between $q-q_G$ and $\nabla\varphi_\texttt{NF}(q)$;
$\delta_c \triangleq 2\pi$;
$s_d(q) \triangleq \exp{(\tau - \frac{\tau {\mu^2}}{(\mu - \varphi_\texttt{NF}(q))^2})}$ is a smooth switch function
that is ``on'' near the destination and ``off'' near the obstacle,
with $\tau\in (0,1)$ being a design parameter and~$\mu$ being the maximum value
of $\varphi_\texttt{NF}(q)$.
Moreover, the variable~$\theta_2(q)$ is the angle to be rotated,
such that the transitional potential fields mimic a point-dipole field, i.e.,
\begin{equation}\label{eq:theta2}
\theta_2(\check{q}) \triangleq s_d(\check{q}) \delta^{\prime}_{\theta}(\check{q}) +
\mathbf{sgn}\big(\delta^{\prime}_{\theta}(\check{q})\big)[1-s_d(\check{q})]\delta^{\prime}_c + \delta^{\prime}_c,
\end{equation}
where~$\check{q}$ is the transformed coordinate of~$q$ in the new coordinate system,
where the goal point~$q_G$ is the origin and the desired orientation~$\theta_G$
is the positive direction of the $x$-axis, i.e.,
$\check{q} \triangleq \mathbf{R}^{-1}(\theta_G) (q - q_G)$,
where~$\mathbf{R}(\theta_G)$ is the rotation matrix associated
with the goal angle~$\theta_G$;
$\delta^{\prime}_{\theta}(\check{q}) \triangleq \theta_{\check{q}}$ is the angle
of the vector~$\check{q}$;
$\delta^{\prime}_c \triangleq \pi$;
$s_d(\cdot)$ is the same as in~\eqref{eq:theta1}.

\begin{definition}\label{def:orientated}
The \emph{oriented} harmonic potential fields are defined as the fields obtained
by transforming the original potential fields through a two-step
rotation in~\eqref{eq:two-step-transformation}, denoted by:
\begin{equation}\label{eq:orientated-harmonic-potential}
 \Upsilon(q) \triangleq -\Gamma(q) \nabla\varphi_\texttt{NF}(q),
\end{equation}
where~$\nabla\varphi_\texttt{NF}(q)$ is the gradient of the original potentials. \hfill $\blacksquare$
\end{definition}

\begin{figure}[t]
  \centering
  \includegraphics[width=0.9\hsize]{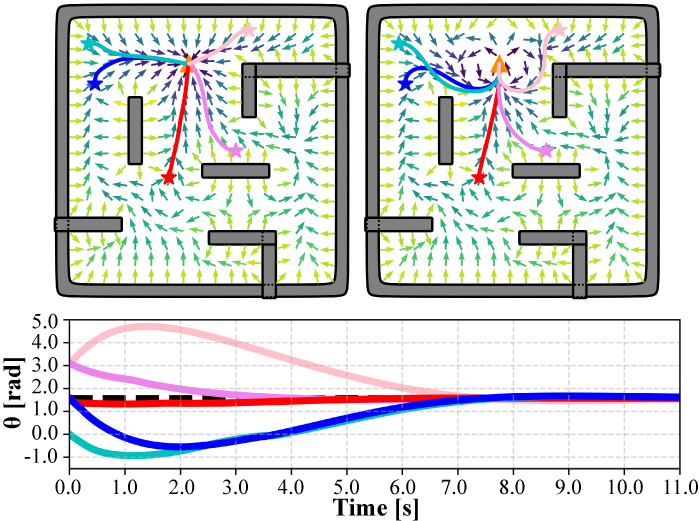}
  \vspace{-0.1in}
  \caption{{Robot trajectories from different initial poses (filled stars)
  to the same goal pose (orange triangle),
  under the original potential fields $-\nabla \varphi_{\texttt{NF}}(q)$
  (\textbf{Top Left})
  and the \emph{oriented} potential fields $\Upsilon(q)$ (\textbf{Top Right})
  via the proposed control law in~\eqref{eq:control}.
  \textbf{Bottom}: The robot orientation~$\theta$ converges to
  the desired value (dashed line), along these trajectories.}}\label{fig:vector-traj}
  \vspace{-0.05in}
\end{figure}

Given the oriented harmonic potential fields,
the integral curves of~$\Upsilon(q)$ correspond to the
trajectories~$\bar{\xi}(t)$ of the autonomous system described by ordinary differential equations:
$\frac{d}{d\,t}\bar{\xi}(t) = \Upsilon(\bar{\xi}(t)) \in \mathbb{R}^{2}$,
with initial value~$\bar{\xi}(t_0)=q_0$.
The convergence property of the integral curves is proven as follows.

\begin{lemma}\label{lemma:transformation-properties}
  {All integral curves of~$\Upsilon(q)$ in~\eqref{eq:orientated-harmonic-potential} converge to $q_G$
  with the desired orientation $\theta_G$ asymptotically, with no collision with any obstacles,
  as long as the conditions in Lemma~\ref{lemma:point-world-nf} hold.}
\end{lemma}
\begin{proof}\label{proof:transformation-properties}
{To begin with, it is proven in Lemma~\ref{lemma:point-world-nf} that~$\varphi_\texttt{NF}(q)$ is a valid navigation function,
it holds that~$\varphi_\texttt{NF}(q)=\mu$ on the boundary of the
workspace~$\partial\mathcal{W}$.
Thus, the switch functions $s_d(q) = 1$ hold both in~\eqref{eq:theta1} and~\eqref{eq:theta2},
and~$\theta_1(q)=\theta_2(q)=\pi$ holds.
In this case, the rotation matrices $\mathbf{R}(\theta_1) = \mathbf{R}(\theta_2)$ are both identity matrices,
thus~$\Gamma(q)=\mathbf{I}$ is an identical transformation
and the transformed field coincides with the original potential field, i.e.,~$-\nabla\varphi_\texttt{NF}(q)$.
Furthermore, since the negated gradient points away from the boundary~$\partial\mathcal{W}$,
the transformed field inherits the property of collision avoidance.}

Secondly, the transformation~$\Gamma(q)$ is nonsingular
as both $\mathbf{R}(\theta_1)$ and $\mathbf{R}(\theta_2)$ are non-singular.
Thus, all critical points of~$\varphi_\texttt{NF}(q)$,
including the global minimum~$q_G$ and saddle points,
retain their properties after the transformation.
In other words, the saddle points of~$\varphi_\texttt{NF}(q)$ remain to be non-degenerate
with an attraction region of measure zero.
Since any integral curve in a closed compact set~$\overline{\mathcal{W}}$ is bounded,
there must be a limit set inside~$\overline{\mathcal{W}}$ to which the integral curves converge.
Lemma~$4$ of~\cite{valbuena2012hybrid} states that no such limit cycles exist
if there are no additional attractors other than~$q_G$.
Therefore, the goal point~$q_G$ is the only attractive component of the
limit set within~$\overline{\mathcal{W}}$, i.e.,
the integral curves of~$\Upsilon(q)$ converge to~$q_G$.

Lastly, it remains to be shown that the asymptotic convergence to~$q_G$ is achieved with the desired
orientation~$\theta_G$.
{Clearly, when~$q$ approaches $q_G$, the switch function~$s_d(q,\tau)$ approaches~$1$,
while $\theta_1(q), \theta_2(q)$ approach ${\delta}_{\theta}(q), \delta^{\prime}_{\theta}(q) + \pi$, respectively.
For brevity, let $\theta_1(q) = {\delta}_{\theta}(q)$ and~$\theta_2(q) = \delta^{\prime}_{\theta}(q) + \pi$}.
Therefore, the transformed potential field~$\Upsilon(q)$ is simplified as:
\begin{equation*}\label{eq:near-destination-transformation}
  \Upsilon(q) =
  \begin{bmatrix}
     \|\nabla\varphi_\texttt{NF}(q)\|\cos\big(\delta^{\prime}_{\theta}(q)+\theta_{q-q_G}\big)\\
     \|\nabla\varphi_\texttt{NF}(q)\|\sin\big(\delta^{\prime}_{\theta}(q)+\theta_{q-q_G}\big)
  \end{bmatrix},
\end{equation*}
where~$\delta^{\prime}_{\theta}(q)\in(-\pi,\pi]$ is defined in~\eqref{eq:theta2}.
Then, the inner product of~$\Upsilon(q)$ and~$q_G-q$ is given by:
$\langle \Upsilon(q),\ q_G-q \rangle =
  -\|\nabla\varphi_\texttt{NF}(q)\| \|q_G-q\| \cos\big(\delta^{\prime}_{\theta}(q)\big)$.
  It should be noted that the integral curves can only converge to the~$q_G$
  in the direction of steepest descent,
i.e., from the region where~$\delta^{\prime}_{\theta}(q)$ approaches~$\pi$.
Since~$\delta^{\prime}_{\theta}(q) = \theta_{\check{q}}$
and~$\check{q} = \mathbf{R}^{-1}(\theta_G) (q - q_G)$,
$\delta^{\prime}_{\theta}(q)$ approaches~$\pi$ implies that~$\theta_{q_G-q}$ approaches~$\theta_G$.
Thus, the integral curves converge to~$q_G$ along the desired orientation~$\theta_G$.
\end{proof}

\begin{figure}[t]
  \centering
  \includegraphics[width=0.85\hsize]{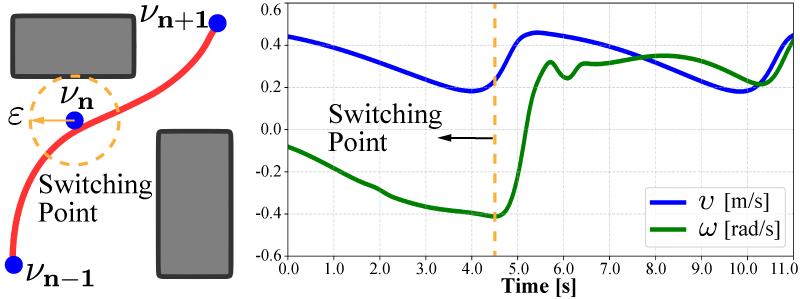}
  \vspace{-0.1in}
  \caption{{
    Illustration of the smooth transition strategy in~\eqref{eq:switch-control} when
    the agent switches from one intermediate waypoint to another.
    The switch function~$\eta_{\epsilon}$ is
    activated once the current waypoint is reached with a margin~$\epsilon$.
    The resulting trajectory and control inputs are all smooth.}
    }
    \label{fig:smooth-control}
  \vspace{-0.05in}
\end{figure}

\subsubsection{Smooth Harmonic-Tracking Controller}
\label{subsubsec:robot-control}
Given the rotated potentials in~\eqref{eq:orientated-harmonic-potential},
{a nonlinear feedback controller can be used for non-holonomic
robots in~\eqref{eq:unicycle} to track its rotated and negated gradient.}
More specifically,
the linear and angular velocity is set as follows:
{\begin{subequations}\label{eq:control}
  \begin{align}
  \upsilon(q,\,q_G) &= k_\upsilon\, \tanh\left(\|q - q_G\|\right);\label{Za}\\
  \omega(\theta,\,\theta_G) & = - k_\omega \left(\theta - \theta_{\Upsilon}\right) +
\upsilon \, [\nabla \theta_{\Upsilon}]^{\intercal} \, J_{\Upsilon} \, \Theta, \label{Zb}
  \end{align}
\end{subequations}}
{where~$q,\,\theta$ are the robot pose and orientation;
$k_\upsilon,\, k_\omega>0$ are controller gains;
$\theta_{\Upsilon}$ is the direction of vector~$\Upsilon \triangleq [\Upsilon_x,\, \Upsilon_y]^{\intercal}$;
$\nabla \theta_{\Upsilon} = [\frac{\partial\, \theta_{\Upsilon}}{\partial\, \Upsilon_x},\,
   \frac{\partial\, \theta_{\Upsilon}}{\partial\, \Upsilon_y}]^{\intercal}$;
$J_{\Upsilon}$ being the~$2 \times 2$ Jacobian matrix of~$\Upsilon$,
whose entries are~$[J_{\Upsilon}]_{ij} = \frac{\partial \Upsilon_i}{\partial q_j}$, for~$i,j\in \{1,2\}$;
and~$\Theta \triangleq [\cos\theta, \, \sin\theta]^{\intercal}$.}
The main idea is to decompose the control objective into two sub-objectives:
(i) the robot orientation is regulated sufficiently fast to~$\theta_{\Upsilon}$ via~\eqref{Zb};
(ii) the robot velocity along this direction is adjusted to reach~$q_G$ via~\eqref{Za}.
The control gains~$k_\upsilon,k_\omega$ should be chosen
  according to the desired safety margin and minimum distance among the obstacles.

{\begin{proposition}\label{proposition:control}
  Under the control law~\eqref{eq:control},
  the robot in~\eqref{eq:unicycle} converges to the goal~$q_G$
  with orientation~$\theta_G$ asymptotically,
  from almost all initial poses in the workspace~$\mathcal{W}$,
  as long as the conditions in Lemma~\ref{lemma:point-world-nf} hold.
\end{proposition}
\begin{proof}
(Sketch) As discussed, the robot system~\eqref{eq:unicycle}
under control law~\eqref{eq:control} can be decomposed into two subsystems
via singular perturbation based on~\cite{khalil2002nonlinear}.
The fast subsystem is equivalent to
$\omega =  - k_\omega \left(\theta - \theta_{\Upsilon}\right) + \dot{\theta}_{\Upsilon}$,
which ensures a global exponential convergence of~$\theta$ to the rotated fields~$\theta_{\Upsilon}$.
Secondly, the slow subsystem is given by
$\dot{q}= \upsilon \frac{\Upsilon}{\|\Upsilon\|}$.
The robot trajectory corresponds one integral curve of the vector field~$\Upsilon(q)$.
As already proven in Lemma~\ref{lemma:point-world-nf} that the integral curves of the vector fields~$\Upsilon(q)$ converge to~$q_G$ with orientation~$\theta_G$,
the slow subsystem converges to the goal pose with the desired orientation.
\end{proof}}

\begin{figure}[t]
  \centering
  \includegraphics[width=0.98\hsize]{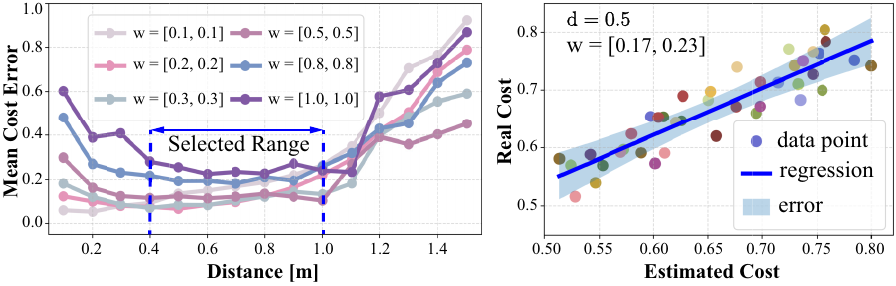}
  \vspace{-0.1in}
  \caption{{
      \textbf{Left:} The difference between the estimated cost by~\eqref{eq:nf-cost}
      and the measured cost under different weights~$\mathbf{w}$,
  for different choice of the distance $d(q_n,\,q_{n+1})$.
  \textbf{Right:} The difference actual with~$40$ samples,
 when~$\mathbf{w}=[0.17,\,0.23]^{\intercal}$ and $d=0.5$.}
  }\label{fig:control-data}
  \vspace{-0.05in}
\end{figure}

  More importantly, since the robot needs to switch along a sequence of waypoints in its
  path~$\mathbf{P}_{\widehat{g}_\ell \rightarrow \widehat{g}_{\ell+1}}$ from~\eqref{eq:path},
  a smooth transition between the intermediate waypoints is crucial
  to ensure both smooth trajectories and control inputs.
  As shown in Fig.~\ref{fig:smooth-control},
  the switching strategy from~$\nu_n = (q_{n}, \theta_{n})$
  to~$\nu_{n+1} = (q_{n+1}, \theta_{n+1})$
  in~$\mathbf{P}_{\widehat{g}_\ell \rightarrow \widehat{g}_{\ell+1}}$ is given by:
  \begin{equation}\label{eq:switch-control}
  u = u(q,\,q_{n}) + \eta_{\varepsilon}\left(\gamma_n(q)\right)\,
 \big(u(q,\,q_{n+1}) - u(q,\,q_{n})\big),
\end{equation}
where~$u$ stands for inputs~$\upsilon$ and~$\omega$;
$\gamma_n(q)\triangleq \|q-q_n\|$; and
$\eta_{\varepsilon}(x) \triangleq \frac{1}{2} \big(1 + \tanh(k_s \, (\varepsilon - \gamma_n(q))) \big)$ is the smooth switch function,
with~$k_s > 0$ being a positive gain,
and~$\varepsilon>0$ the switching margin around~$q_{n}$,
which can be tuned according to how dense the obstacles are located
to ensure safety during transition.
Namely, a smaller~$\varepsilon$ should be chosen when obstacles
are close to the switching point.
Last but not least,
the control cost under the proposed control law in~\eqref{eq:control}
to drive the robot from waypoint~$\nu_n$ to $\nu_{n+1}$ is estimated by:
\begin{equation}\label{eq:nf-cost}
  \begin{aligned}
  &\gamma(\nu_n,\,\nu_{n+1}) \triangleq
    d(q_n,\,q_{n+1}) + \mathbf{w} \, {Q}_{\Upsilon}^{\intercal}(\theta_n,\theta_{n+1}),
 \end{aligned}
\end{equation}
where the first part~$d(q_n,\,q_{n+1})=\|q_{n}-q_{n+1}\|$ measures the straight-line distance;
the second part approximates the rotation cost:
(i) $\mathbf{w}\triangleq [w_1,w_2]$ are the weighting parameters with~$w_1,w_2$ $>0$;
(ii) ${Q}_{\Upsilon}(\theta_n, \theta_{n+1})\triangleq [|\theta_{\Upsilon(q_n)} - \theta_n|,\, |\theta_{q_{n+1}-q_{n}} - \theta_{n+1}|]$,
{where the first item estimates the rotation cost
between robot orientation and the rotated gradient at the starting pose;
the second item estimates the steering cost
between the vector~$q_{n+1} - q_n$ and goal orientation.}
Consequently, it can be used to compute the edge cost of the HT~$\mathcal{T}_{\widehat{g}_\ell\rightarrow \widehat{g}_{\ell+1}}$
in~\eqref{eq:tree}, yielding a more accurate estimate of the navigation cost from $\widehat{g}_\ell$ to $\widehat{g}_{\ell+1}$.
In turn, the actual cost within the navigation map $\mathcal{G}$ is given by:
  $d(\widehat{g},\, \widehat{g}')=\sum^{N-1}_{n=0}\gamma(\nu_n,\,\nu_{n+1})$,
where~$\mathbf{P}_{\widehat{g}\rightarrow \widehat{g}'}=\nu_0\nu_1\cdots \nu_N$
is the shortest path from~$\widehat{g}$ to~$\widehat{g}'$
in the associated HT $\mathcal{T}_{\widehat{g}\rightarrow \widehat{g}'}$.
{\begin{remark}\label{remark:compare-purging}
The control cost in~\eqref{eq:nf-cost}
is estimated through a weighted combination of the translation cost
and the steering cost.
As shown in Fig.~\ref{fig:control-data},
this estimation approaches the actual measured cost
when the distance between consecutive waypoints are selected properly.
\hfill $\blacksquare$
\end{remark}}

\subsubsection{Hybrid Execution of the Initial Plan}
\label{subsubsec:execute}
\begin{figure}[t]
  \centering
  \includegraphics[width=0.98\hsize]{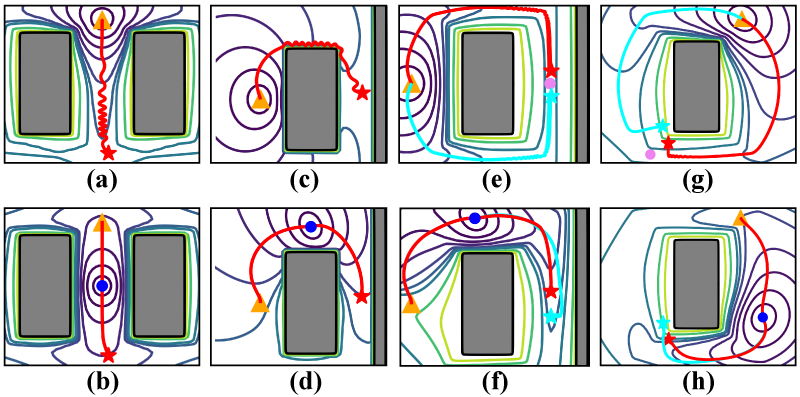}
  \vspace{-0.1in}
  \caption{
    Illustration of how intermediate waypoints can reduce oscillations and long detours,
    under different initial poses (solid stars) and goal poses (orange triangles).
    (i) Oscillations appear as the robot traverses narrow passage
    in \textbf{(a)} or near the obstacle boundary in \textbf{(c)}.
   Intermediate waypoints (blue dots) are introduced in \textbf{(b)} and \textbf{(d)},
   yielding smooth trajectories.
  (ii) Divergent trajectories appear for two close-by initial poses (red and green stars)
  due to their proximity to the saddle point (violet dot) in \textbf{(e)} and \textbf{(g)}.
  Intermediate waypoints (blue dots) are introduced in \textbf{(f)} and \textbf{(h)},
  yielding nearly identical trajectories.
  }\label{fig:saddle-valley}
  \vspace{-0.05in}
\end{figure}
The initial plan $\widehat{\mathbf{g}}=\widehat{g}_1\widehat{g}_2\cdots\widehat{g}_L(\widehat{g}_{L+1}\widehat{g}_{L+2}\cdots\widehat{g}_{L+H})^\omega$
in \eqref{eq:plan} can be executed via the following hybrid strategy:
(i) staring from~$\ell=1$, determine the next goal region~$\widehat{g}_{\ell+1}$;
(ii) construct the HT as $\mathcal{T}_{\widehat{g}_{\ell} \rightarrow \widehat{g}_{\ell+1}}$
and determine the shortest path
$\mathbf{P}_{\widehat{g}_{\ell} \rightarrow \widehat{g}_{\ell+1}}=\nu_0\cdots\nu_{N-1}\nu_G$ in~\eqref{eq:path};
  (iii) starting from the first vertex~$n=0$,
once the robot enters the small vicinity of vertex~$\nu_n$,
it switches to the next mode of traversing the subsequent
intermediate edge $(\nu_n,\,\nu_{n+1})$.
The associated controller~\eqref{eq:control} is activated with the goal
pose~$\nu_{n+1}$ and potential~$\Upsilon_{\nu_n \rightarrow \nu_{n+1}}(q)$.
The controller remains active until the robot enters the small vicinity of~$\nu_{n+1}$.
This execution repeats until the vertex $\nu_G$, i.e., $\widehat{g}_{\ell+1}$, is reached,
after which the index~$\ell$ is incremented by $1$ and returns to step (i).
This procedure could be repeated until the last state of plan suffix~$\hat{s}_{L+H}$ is reached,
if the environment remains static.
However, since the environment is only partially known,
more obstacles would be detected during execution and added to the environment model.
which might invalidate not only the current navigation map but also the underlying harmonic potentials.

\begin{remark}\label{remark:compare-NF}
  {
    The design of intermediate waypoints between subtasks can alleviate certain
    limitations of the classic navigation functions from~\cite{koditschek1987exact,rimon1990exact}
    for robotic navigation.
    As shown in Fig.~\ref{fig:saddle-valley},
    when the robotic system is geometrically and dynamically constrained,
    it can \emph{not} follow the negated gradient perfectly.
    Consequently, oscillations appear along narrow passages and close
    to boundary of obstacles,
    where the underlying gradient changes directions abruptly.
    As also emphasized in~\cite{rimon1990exact,loizou2022mobile},
    the phenomenon of ``vanishing" valleys and the resulting oscillations
    often occur when the parameter~$\lambda$ is set too large.
    Moreover, when two initial poses are located close to saddle points,
    the resulting trajectories can be diverging in opposite directions.
    To overcome these issues,
    the introduced intermediate waypoints can decompose the
    long path into shorter segments
    of which the associated potentials are often easier to track.}
  \hfill $\blacksquare$
  \end{remark}




\subsection{Online Adaptation}\label{subsec:online}

As the robot detects more obstacles during execution,
the estimated obstacles, the HTs and the underlying harmonic potentials are all updated as follows.

\subsubsection{Adaptive Obstacle Estimation}
\label{subsubsec:online-squircle}
\begin{figure}[t]
  \centering
  \includegraphics[width=0.98\hsize]{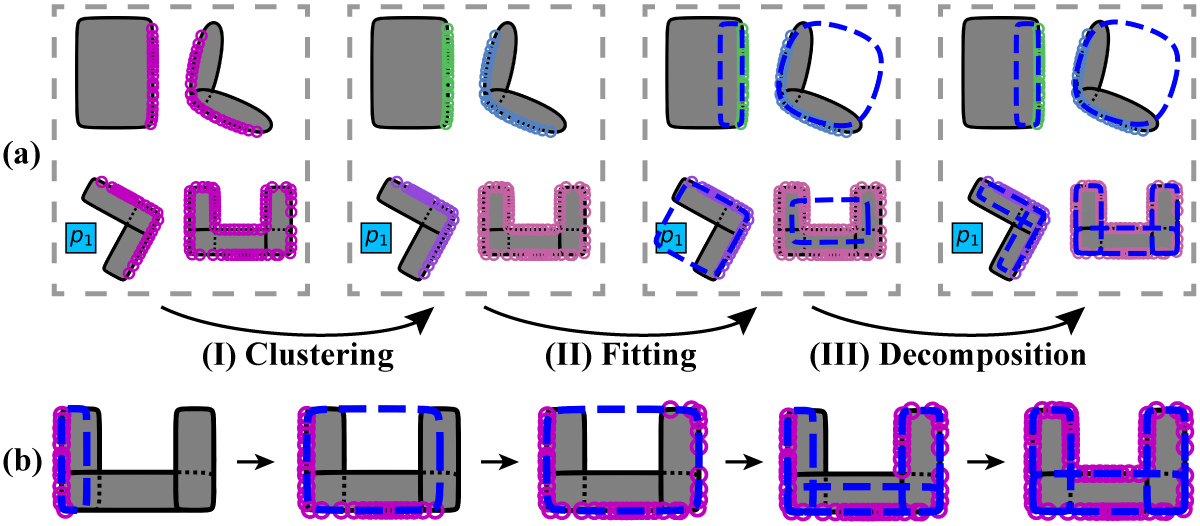}
  \vspace{-0.1in}
  \caption{
  \textbf{(a)} Estimation of the obstacles
      (in blue dashed lines) in Sec.~\ref{subsubsec:online-squircle},
      which includes clustering, fitting and decomposition;
      \textbf{(b)} Online adjustment
      as more measurements (pink circles) are gathered.
      The cluster is ultimately decomposed into three segments, which are fitted to three squircles
      when the error exceeds the desired threshold for fitting a single squircle.
  }\label{fig:adapt_squircle_esti}
  \vspace{-0.05in}
\end{figure}
  As discussed in Sec.~\ref{subsubsec:init-nf} and shown in Fig.~\ref{fig:squircle_esti},
  the estimation of squircle obstacles can be inaccurate and uncertain
  when the measurements are noisy and partial.
  Instead of simply relying on Assumption~\ref{assump:accurate},
  the obstacle estimation can be adapted online
  as more measurements are gathered, as shown in Fig.~\ref{fig:adapt_squircle_esti}.
  In particular, this process involves three steps that are executed at a desired frequency:
  (i) clustering the measurements to identify potential obstacle regions;
  (ii) fitting the geometric models to these clusters for improved accuracy; and
  (iii) decomposing the fitted models to refine the obstacle representation.

  More specifically, \textbf{Clustering}: Given the set of measurements $\mathbf{D}_{t}$ at the current time step $t$,
   the data points are clustered into $K$ clusters $\mathbf{D}_{t} = \{\mathbf{D}_{t,1}, \cdots, \mathbf{D}_{t,K}\}$
   by their relative distances and curvature.
  \textbf{Fitting}: The data points in each cluster $\mathbf{D}_{t,k}$
  are then fitted to a squircle $\hat{\mathcal{O}}_{t,k}$
  using nonlinear least squares methods~\cite{gavin2019levenberg}, $\forall k =1, \cdots, K$.
  If the fitting error $E_k$ falls below a pre-defined threshold,
  the estimated model $\hat{\mathcal{O}}_{t,k}$ is accepted,
  by which either a new squircle is added to the forest world
  or the geometric parameters of an existing squircle is updated.
  Since the fitting process is nonlinear,
  multiple estimated models with comparable fitting errors may exist.
  In such cases, the model that is most \emph{conservative} is selected.
  For instance, if only one side of a rectangular squircle is detected,
  a minimum width or height is assumed for the missing side,
  as shown in Fig~\ref{fig:adapt_squircle_esti}.
  \textbf{Decomposition}: If the fitting error $E_k$ is larger than the threshold,
  it indicates that this cluster $\mathbf{D}_{t,k}$ might not correspond to a single squircle
  and should be decomposed into multiple segments $\mathbf{D}_{t,k}=\{\mathbf{S}_{t,1}, \cdots, \mathbf{S}_{t,L}\}$,
  each of which is then fitted to a squircle model $\hat{\mathcal{O}}_{t,\ell}$, $\forall \ell=1, \cdots, L$.
  Namely, the curvature of each point $p_j \in \mathbf{D}_{t,k}$ is captured locally based on its neighboring points.
  Then, the cluster $\mathbf{D}_{t,k}$ is divided into several segments
  based on the changes in curvatures between consecutive points.
  Moreover, if the estimated squircle overlaps with the robot or the regions of interest,
  the cluster should also be further decomposed to avoid these regions.
  This process is repeated until all segments are fitted to squircles
  with the accepted accuracy.

More technical details can be found in the supplementary
  files in~\cite{wang2023harmonic}.
  Some illustrations are shown in Fig.~\ref{fig:adapt_squircle_esti},
  with more numerical examples provided in Sec.~\ref{sec:experiments}.
  Via this adaptive scheme, both the obstacle parameters
  and the structure of the forest world can be updated online.

\subsubsection{Incremental Update of Harmonic Potentials}
\label{subsubsec:online-nf}

The classic methods as proposed in~\cite{loizou2022mobile, li2018navigation, rimon1990exact}
construct the underlying navigation function at once by taking into
account \emph{all} obstacles in the  workspace.
However, this means that the whole process has to be repeated from scratch
each time an additional obstacle is added.
A more intuitive approach would be to incrementally update different parts of
the computation process by storing and re-using some of the intermediate results.
Consider the following \emph{two} cases as shown in Fig.~\ref{fig:nf-update}:
(i) independent stars are added to existing world of stars;
and (ii) overlapping stars are added to existing forests of stars.

\textbf{Independent Stars}:
To begin with, consider the simple case that a set of non-overlapping star obstacles,
denoted by~$\mathcal{O}_k$, $k=1,\cdots,M$,
is added to an empty workspace~$\mathcal{W}=\mathcal{O}_0$ in sequence.
As discussed in Sec.~\ref{subsubsec:init-nf}, the star-to-sphere transformation
is constructed via the translated scaling map
in~\eqref{eq:translated-scaling-map}.

\begin{figure}[t]
  \centering
  \includegraphics[width=0.8\hsize]{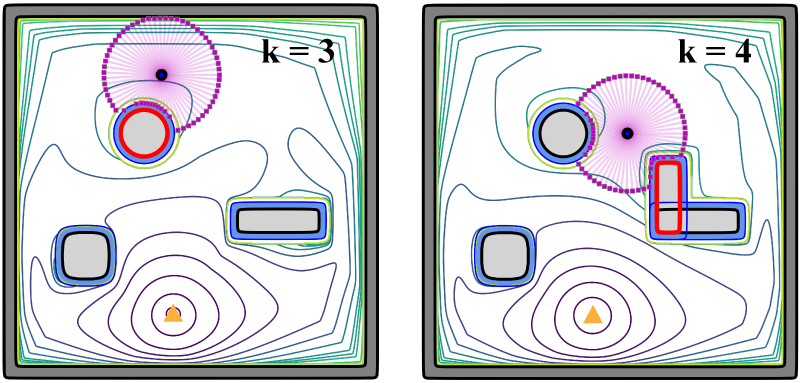}
  \vspace{-0.1in}
  \caption{Iterative update of the harmonic potentials
      as new obstacles (red lines) are added:
      an \emph{independent} obstacle  for $k=3$ (\textbf{Left});
      and an \emph{overlapping} obstacle for $k=4$ (\textbf{Right}).
  }\label{fig:nf-update}
  \vspace{-0.05in}
\end{figure}

\begin{definition}\label{def:online-omitted-product}
  {The online omitted product $\overline{\beta}_{i}^{k}(q)$
of star-to-sphere transformation for obstacle~$\mathcal{O}_i$
after adding the $k$-th obstacle,
for~$i=0,1,\cdots,k$ and for $k=1,\cdots,M$,
is a real-valued function defined for the star world $\mathcal{S}$ as:
$\overline{\beta}_{i}^{k}(q) \triangleq \prod_{j=0,j\ne i}^k \beta_j(q)$,
where $\beta_j(q)$ is the obstacle function defined in~\eqref{eq:squircle}.
 \hfill $\blacksquare$}
\end{definition}
\begin{definition}\label{def:online-analytic-switch}
{The online analytic switch $s_i^{k}(q)$ of star-to-sphere transformation for obstacle $\mathcal{O}_i$,
after adding the $k$-th obstacle, for $i=0,1,\cdots, k$ and for $k=1,\cdots, M$, is a real-valued function defined for a star world~$\mathcal{S}$ as:
$$s_i^{k}(q) \triangleq \frac{\gamma_G(q)\overline{\beta}_{i}^{k}(q)}
{\lambda_{k}\beta_{i}(q)+\gamma_G(q)\overline{\beta}_{i}^{k}(q)},$$
where~$\gamma_G(q)$ is the distance-to-goal function;
$\lambda_{k}$ is a positive parameter associated with~$k$;
$\overline{\beta}_{i}^{k}(q)$ is the online omitted product;
and $\beta_i(q)$ is the obstacle function defined in~\eqref{eq:squircle}.
 \hfill $\blacksquare$}
\end{definition}
Then, the transformation from the star world to the sphere world after
adding the $k$-th star is restated as:
\begin{equation}\label{eq:online-tf-s-m}
  \Phi_{\mathcal{S}\rightarrow \mathcal{M}}^{k}(q) \triangleq {\rm id}(q) + \sum_{i=0}^k s_i^{k}(q)\, \big(v_i(q)-1\big)\,(q-q_i),
\end{equation}
where~$s_i^{k}(q)$ is the online analytic switch associated with both~$i$ and~$k$;
$v_i(q)$ is the scaling factor defined in~\eqref{eq:scaling};
and~$q_i$ is the geometric center of the obstacle~$\mathcal{O}_i$,
for~$i=0,1,\cdots, k$.
Moreover, the following intermediate variables are defined:
\begin{equation}\label{eq:restate-scaling-map}
  \mathbf{F}_i^{k}(q) \triangleq s_i^{k}(q)\,\mathbf{r}_i(q),
\end{equation}
where~$\mathbf{r}_i \triangleq \big(v_i(q)-1\big)(q-q_i)$.
Consequently, the transformation in~\eqref{eq:online-tf-s-m} is adjusted to:
$\Phi_{\mathcal{S}\rightarrow \mathcal{M}}^{k}(q) = {\rm id}(q) + \sum_{i=0}^k
\mathbf{F}_i^{k}(q)$.
Note that for~$k=0$, the function~$\mathbf{F}_0^{0}(q)$ for the empty workspace is computed using the initial transformation via~\eqref{eq:tf-s-m}.
Afterwards, each time a new~$(k+1)$-th obstacle is added,
$\mathbf{F}_i^{k}(q)$ should be updated to~$\mathbf{F}_i^{k+1}(q)$, $\forall i=0,1,\cdots, k$,
and~$\mathbf{F}_{k+1}^{k+1}(q)$ for the added obstacle should also be constructed.
To compute these variables efficiently, a recursive method is proposed here.
For the ease of notation,
define the multi-variate auxiliary function useful for the recursive computation as follows:
{\begin{equation}\label{eq:general-iterative-function}
  \Psi_\texttt{a}(\mathbf{F},\, \mathbf{r},\, \mathbf{\Tilde{r}},\, \alpha) \triangleq \frac{\|\mathbf{F}\|}
      {\alpha \, \|\mathbf{r}\| +(1-\alpha)\,\|\mathbf{F}\|} \,\mathbf{\Tilde{r}},
\end{equation}
  where $\mathbf{F}, \mathbf{r}, \mathbf{\Tilde{r}}$ are vectors;
  $\alpha\in (0,\,1)$ is a scalar variable;
and $\alpha \, \|\mathbf{r}\| +(1-\alpha)\,\|\mathbf{F}\| \neq 0$.}
Then, the function~$\mathbf{F}_{0}^{k+1}(q)$ associated with the workspace is updated from~$\mathbf{F}_{0}^{k}(q)$ as follows:
\begin{equation}\label{eq:workspace-update}
 \mathbf{F}_{0}^{k+1}(q) = \Psi_\texttt{a}\left(\mathbf{F}_0^{k}(q),\,\mathbf{r}_0(q),\,\mathbf{r}_0(q),\,\alpha^{k}(q)\right),
\end{equation}
where~$\mathbf{r}_0(q)$ is defined in~\eqref{eq:restate-scaling-map} for the workspace;
the variable $\alpha^{k}(q) \triangleq \lambda_{k+1}/\big(\lambda_{k}\, \beta_{k+1}(q)\big)$;
$\lambda_{k}$ and~$\lambda_{k+1}$ are positive parameters defined in~\eqref{def:online-analytic-switch},
for $k=0,1,\cdots,M$.
\begin{lemma}\label{lemma:update-workspace}
  {Each time an independent obstacle $\mathcal{O}_{k+1}$
    is added to the workspace, the following holds:}

{(i) the online omitted product for $\mathcal{O}_0$
is given by $\overline{\beta}_{0}^{k+1}(q) = \overline{\beta}_{0}^{k}(q) \, \beta_{k+1}(q)$;}

{(ii) the online analytic switch for $\mathcal{O}_0$
  is given by
  $s_0^{k+1}(q) = \frac{s_0^{k}(q)}{\alpha^{k}(q)\, (1-s_0^{k}(q)) + s_0^{k}(q)}$,
  where $\alpha^{k}(q)=\lambda_{k+1}/\big(\lambda_{k}\, \beta_{k+1}(q)\big)$.
}
\end{lemma}
\begin{proof}
  {(Omitted) Available in the supplementary files.}
\end{proof}
\begin{proposition}\label{proposition:update-workspace}
  {Each time an independent obstacle $\mathcal{O}_{k+1}$ is added to the workspace,
  the recursive calculation of~$\mathbf{F}_{0}^{k+1}(q)$
  in~\eqref{eq:workspace-update} for $\mathcal{O}_0$
  is valid for iteration~$k=0,1,\cdots, M-1$.}
\end{proposition}
\begin{proof}
{(Sketch) The proof is done by induction.
Namely, the initial value of~$\mathbf{F}_{0}^{k}(q)$ is given by
$\mathbf{F}_{0}^{0}(q) = s_0^{0}(q)\,\mathbf{r}_0(q)$ with $s_0^{0}(q)=\frac{\gamma_G(q)}{\lambda_{0}\beta_{0}(q)+\gamma_G(q)}$,
which satisfies~\eqref{eq:restate-scaling-map}.
Now assume that $\mathbf{F}_{0}^{\ell}(q)=s_0^{\ell}(q)\,\mathbf{r}_0(q)$
and $s_0^{\ell}(q)=\frac{\gamma_G(q)\overline{\beta}_0^{\ell}(q)}{\lambda_{\ell}\beta_0(q) + \gamma_G(q)\overline{\beta}_0^{\ell}(q)}$ hold
for an integer $\ell\geq0$.
Then, $\mathbf{F}_{0}^{\ell+1}(q)$ is derived via~\eqref{eq:workspace-update} as:
\begin{equation*}
\mathbf{F}_{0}^{\ell+1}(q)
=\frac{\|s_0^{\ell}(q)\,\mathbf{r}_0(q)\| \,\mathbf{r}_0(q)}{\alpha^{\ell}(q) \, \|\mathbf{r}_0(q)\|+\big(1-\alpha^{\ell}(q)\big)\,\|s_0^{\ell}(q)\,\mathbf{r}_0(q)\|}.
\end{equation*}
Given~\eqref{eq:squircle}, it holds that
$\beta_0(q) \geq 0$ and $\overline{\beta}_0^{\ell}(q) \geq 0$,
which implies that $s_0^{\ell}(q) \in [0,\,1]$.
Therefore, $\mathbf{F}_{0}^{\ell+1}(q)$ is simplified as
$\mathbf{F}_{0}^{\ell+1}(q)=\frac{s_0^{\ell}(q) \,\mathbf{r}_0(q)}{\alpha^{\ell}(q)+(1-\alpha^{\ell}(q))\,s_0^{\ell}(q)}$
Further, Lemma~\ref{lemma:update-workspace} implies
that $s_0^{\ell}(q) = \frac{\alpha^{\ell}(q)\,s_0^{\ell+1}(q)}{(\alpha^{\ell}(q)-1)\,s_0^{\ell+1}(q) + 1}$.
Combined with $\mathbf{F}_{0}^{\ell+1}(q)$,
it yields that
$\mathbf{F}_{0}^{\ell+1}(q)=s_0^{\ell+1}(q) \, \mathbf{r}_0(q),$
which is consistent with~\eqref{eq:restate-scaling-map},
thus concluding the proof.}
\end{proof}

Furthermore, the function~$\mathbf{F}_{i+1}^{k+1}(q)$ for each
inner obstacle~$i=0,1,\cdots,k$ is calculated from~$\mathbf{F}_{0}^{k+1}(q)$ iteratively by:
\begin{equation}\label{eq:obstacle-update}
 \mathbf{F}_{i+1}^{k+1}(q) = \Psi_{\texttt{a}}\left(\mathbf{F}_{i}^{k+1}(q),\,\mathbf{r}_{i}(q),\,\mathbf{r}_{i+1}(q),\,\alpha_{i}(q)\right),
\end{equation}
where~$\mathbf{r}_{i}(q)$ and~$\mathbf{r}_{i+1}(q)$ are defined in~\eqref{eq:restate-scaling-map}
for the~$i$-th and~$(i+1)$-th obstacles;
$\alpha_{i}(q)=(\beta_{i+1}(q))^2/(\beta_{i}(q))^2$.
\begin{lemma}\label{lemma:update-obstacle}
  {Each time an independent obstacle $\mathcal{O}_{k+1}$
    is added to the workspace, the following holds:}

{(i)
  the online omitted product for $\mathcal{O}_{i+1}$
  is given by $\overline{\beta}_{i+1}^{k+1}(q) = \overline{\beta}_{i}^{k+1}(q) \, \frac{\beta_{i}(q)}{\beta_{i+1}(q)}$;}

{(ii)
  the online analytic switch for $\mathcal{O}_{i+1}$
is given by
$s_{i+1}^{k+1}(q) = \frac{s_i^{k+1}(q)}{\alpha_{i}(q)\, (1-s_i^{k+1}(q))
  + s_i^{k+1}(q)}$, where $\alpha_{i}(q)=(\beta_{i+1}(q))^2/(\beta_{i}(q))^2$.}
\end{lemma}
\begin{proof}
  {(Omitted) Available in the supplementary files.}
\end{proof}
\begin{proposition}\label{proposition:update-obstacle}
  {Each time an independent obstacle $\mathcal{O}_{k+1}$ is added to the workspace,
  the recursive calculation of~$\mathbf{F}_{i+1}^{k+1}(q)$
  in~\eqref{eq:workspace-update} for $\mathcal{O}_{i+1}$
  is valid for iteration~$i=0,1,\cdots, k$.}
\end{proposition}
\begin{proof}
{(Sketch) Similar to Proposition~\ref{proposition:update-workspace}, the proof is done by induction.
The initial value of~$\mathbf{F}_{i}^{k+1}(q)$ is given by
$\mathbf{F}_{0}^{k+1}(q) = s_0^{k+1}(q)\,\mathbf{r}_0(q)$,
which fulfills~\eqref{eq:restate-scaling-map}.
Now assume that~$\mathbf{F}_{\ell}^{k+1}(q)=s_{\ell}^{k+1}(q)\,\mathbf{r}_{\ell}(q)$ with $s_{\ell}^{k+1}(q)=\frac{\gamma_G(q) \overline{\beta}_{\ell}^{k+1}(q)}{\gamma_G(q)\overline{\beta}_{\ell}^{k+1}(q)+ \lambda_{k+1}\beta_{\ell}(q)}$ holds for an integer~$\ell \geq0$.
Then, $\mathbf{F}_{\ell+1}^{k+1}(q)$ is derived via~\eqref{eq:workspace-update} as:
\begin{equation*}
\mathbf{F}_{\ell+1}^{k+1}(q)
=\frac{\|s_{\ell}^{k+1}(q)\,\mathbf{r}_{\ell}(q)\| \,\mathbf{r}_{\ell+1}(q)}{\alpha_{\ell}(q) \, \|\mathbf{r}_{\ell}(q)\|+\big(1-\alpha_{\ell}(q)\big)\,\|s_{\ell}^{k+1}(q)\,\mathbf{r}_{\ell}(q)\|}.
\end{equation*}
Since $s_{\ell}^{k+1}(q) \in [0,1]$ holds,
it can be simplified to~$\mathbf{F}_{\ell+1}^{k+1}(q)
=\frac{s_{\ell}^{k+1}(q) \,\mathbf{r}_{\ell+1}(q)}{\alpha_{\ell}(q) \, +(1-\alpha_{\ell}(q))\,s_{\ell}^{k+1}(q)}$.
Furthermore, Lemma~\ref{lemma:update-obstacle} implies that
$s_{\ell}^{k+1}(q) = \frac{s_{\ell+1}^{k+1}(q)}{\alpha_{\ell}(q)\, (1-s_{\ell+1}^{k+1}(q)) + s_{\ell+1}^{k+1}(q)}$.
Combined with $\mathbf{F}_{\ell+1}^{k+1}(q)$,
it yields that
$\mathbf{F}_{\ell+1}^{k+1}(q)=s_{\ell+1}^{k+1}(q) \, \mathbf{r}_{\ell+1}(q),$
which is consistent with~\eqref{eq:restate-scaling-map}, thus concluding the proof.}
\end{proof}

\begin{remark}\label{remark:independent-stars}
 Note that~$\mathbf{F}_0^{k+1}(q)$ in~\eqref{eq:workspace-update} is iterated over the index~$k$ with initial value~$\mathbf{F}^{0}_0(q)$
 while~$\mathbf{F}_{i+1}^{k+1}(q)$ in~\eqref{eq:obstacle-update} is iterated over the index~$i$ with initial value~$\mathbf{F}_0^{k+1}(q)$.
The reason is that as the $(k+1)$-th obstacle is added to the existing star world,
$\mathbf{F}_0^{k}$ for the outer boundary should be updated to~$\mathbf{F}_0^{k+1}$ first,
and then the transformation~$\mathbf{F}_{i+1}^{k+1}$ for each inner obstacle~$i=0,1,\cdots, k$
is constructed incrementally based on~$\mathbf{F}_0^{k+1}$.
\hfill $\blacksquare$
\end{remark}

\textbf{Overlapping Stars}: Secondly, it is also possible that the new obstacle is overlapping
with an existing obstacle.
More precisely, assume that a sequence of star obstacles, denoted
by $\{\mathcal{O}_k, k=1,\cdots, N\}$,
is added to a star world and overlapping with one of the existing obstacles as a leaf.
As discussed in Sec.~\ref{subsubsec:init-nf} for the static scenario,
the purging transformation~\eqref{eq:purging-tf} for the leaf obstacles
is also constructed by ray scaling process\eqref{eq:purging-scaling}.
To be specific, the purging transformation for the~$k$-th obstacle in one of the tree is restated as:
\begin{equation}\label{eq:online-purging}
  f_{k}(q) \triangleq {\rm id}(q) + \sigma_k(q)\, \big(v_k(q)-1\big)\,(q-p_k),
\end{equation}
where~$\sigma_k(q)$ is the online analytic switch associated with~$k$;
$v_k(q)$ is the scaling factor defined in~\eqref{eq:purging-scaling};
and~$p_k$ is the common center between the obstacle~$\mathcal{O}_k$ and its parent~$\mathcal{O}_k^{\star}$.
\begin{definition}\label{def:online-omitted-product-purging}
{The online omitted product $\overline{\beta}_{k}(q)$ of leaf-purging transformation
 for obstacle $\mathcal{O}_k$, after adding the $k$-th obstacle for $k=1,\cdots, N$,
 is defined on the forest world $\mathcal{F}$ as:
\begin{equation}\label{eq:online-omitted-product-purging}
\overline{\beta}_{k}(q) \triangleq \Big(\prod_{j=0,j \neq k^{\star}}^{k-1}\beta_j(q)\Big)
\Big(\prod_{j \in \mathcal{L}\backslash \{k\} } \beta_j(q)\Big)\Tilde{\beta}_k(q),
\end{equation}
where~$\Tilde{\beta}_k(q)$ is defined in~\eqref{eq:purging};
$\beta_j(q)$ is the obstacle function defined in~\eqref{eq:squircle};
and $\mathcal{L}$ is the set of indices for all leaves in $\mathcal{F}$.
 \hfill $\blacksquare$}
\end{definition}
\begin{definition}\label{def:online-analytic-switch-purging}
{The online analytic switch $\sigma_{k}(q)$ of leaf-purging transformation for obstacle $\mathcal{O}_k$,
after adding the $k$-th obstacle for $k=1,\cdots, N$,
is defined on the forest world $\mathcal{F}$ as:
\begin{equation}
 \sigma_{k}(q)\triangleq \frac{\gamma_G(q)\overline{\beta}_{k}(q)}{\xi_k\beta_k(q) + \gamma_G(q)\overline{\beta}_{k}(q)},
\end{equation}
where $\gamma_G(q)$ is the distance-to-goal function;
$\xi_k$ is a positive parameter associated with $k$;
$\overline{\beta}_{k}(q)$ is the online omitted product;
and $\beta_k(q)$ is the obstacle function.
 \hfill $\blacksquare$}
\end{definition}
Similar to~\eqref{eq:restate-scaling-map}, the intermediate variables are defined as:
\begin{equation}\label{eq:restate-purging}
  \mathbf{H}_k(q) \triangleq \sigma_k(q)\,\mathbf{r}_k(q),
\end{equation}
where~$\mathbf{r}_k(q) \triangleq \big(v_k(q)-1\big)(q-p_k)$.
Consequently, the transformation from forest world to star world
after adding the~$k$-th star obstacle is given by:
\begin{equation}\label{eq:online-tf-f-s}
  \Phi_{\mathcal{F}\rightarrow \mathcal{S}}^{k}(q) \triangleq \mathbf{H}_1 \circ
  \mathbf{H} \circ \cdots\circ \mathbf{H}_k(q).
\end{equation}
Note that for~$k=1$, the new obstacle forms a new tree of stars
and the transformation function~$\mathbf{H}_1(q)$ is computed via~\eqref{eq:purging}.
Afterwards, each time a new~$(k+1)$-th obstacle is added,
the transformation~\eqref{eq:online-tf-f-s} should be updated as follows:
\begin{equation}\label{eq:online-tf-f-s-update}
  \Phi_{\mathcal{F}\rightarrow \mathcal{S}}^{k+1}(q) = \Phi_{\mathcal{F}\rightarrow \mathcal{S}}^{k} \circ \mathbf{H}_{k+1}(q),
\end{equation}
where~$\mathbf{H}_{k+1}(q)$ is calculated from~$\mathbf{H}_{k}(q)$ iteratively by:
\begin{equation}\label{eq:successive-tree-update}
 \mathbf{H}_{k+1}(q) = \Psi_{\texttt{a}}\left(\mathbf{H}_{k}(q),\,\mathbf{r}_{k}(q),\,\mathbf{r}_{k+1}(q),\,\alpha_{k}(q)\right),
\end{equation}
where~$\Psi_{\texttt{a}}(\cdot)$ is defined in~\eqref{eq:general-iterative-function};
$\mathbf{r}_{k+1}(q)$ and $\mathbf{r}_{k}(q)$ are defined in~\eqref{eq:restate-purging};
and~$\alpha_{k}(q)$ is another intermediate variable, given by:
$$
\alpha_{k}(q)=\zeta_k \,\frac{\beta_{k+1}(q)\, \Tilde{\beta}_{k}(q)\, \beta_{{(k+1)}^{\star}}(q)}
{(\beta_{k}(q))^3\,\Tilde{\beta}_{k+1}(q) \, \beta_{{k}^{\star}}(q)},
$$
where~$\zeta_k=\xi_{k+1}/\xi_{k}$ with~$\xi_{k}$ and~$\xi_{k+1}$ being the design parameters from~\eqref{eq:purging-tf};
$\Tilde{\beta}_{k+1}(q)$ and $\Tilde{\beta}_{k}(q)$ are defined as in~\eqref{eq:purging}.
\begin{lemma}\label{lemma:update-obstacle-purging}
  {Each time an obstacle $\mathcal{O}_{k+1}$ is added to the
    workspa-ce and overlapping with an existing obstacle,
    the following holds:}

  {(i) the online omitted product for $\mathcal{O}_{k+1}$ is given by
 $\overline{\beta}_{k+1}(q) = \overline{\beta}_{k}(q) \, \frac{(\beta_{k}(q))^2\,\Tilde{\beta}_{k+1}(q) \, \beta_{{k}^{\star}}(q)}
{\Tilde{\beta}_{k}(q)\, \beta_{{(k+1)}^{\star}}(q)}$;}

  {(ii) the online analytic switch for $\mathcal{O}_{k+1}$
  is given by:
$\sigma_{k+1}(q) = \frac{\sigma_{k}(q)}{\alpha_{k}(q)(1-\sigma_{k}(q)) + \sigma_{k}(q)}$,
where $\alpha_{k}(q)=\frac{\xi_{k+1}\, \beta_{k+1}(q)\, \Tilde{\beta}_{k}(q)\, \beta_{{(k+1)}^{\star}}(q)}
{\xi_{k}\, (\beta_{k}(q))^3\,\Tilde{\beta}_{k+1}(q) \, \beta_{{k}^{\star}}(q)}$.}
\end{lemma}
\begin{proof}
  {(Omitted) Available in the supplementary files.}
\end{proof}
\begin{proposition}\label{proposition:update-obstacle-purging}
  {Each time an obstacle $\mathcal{O}_{k+1}$ is added to the workspace
  and overlapping with an existing obstacle,
  the recursive calculation of~$H_{k+1}(q)$
  in~\eqref{eq:successive-tree-update} is valid, for each new leaf obstacle~$k=0,1,\cdots, N-1$.}
\end{proposition}
\begin{proof}
  {Similar to Proposition~\ref{proposition:update-obstacle},
    the proof is done by induction.
The initial value of~$\mathbf{H}_{k}(q)$ is given by
$\mathbf{H}_{1}(q) = \sigma_{1}(q)\,\mathbf{r}_1(q)$,
which satisfies~\eqref{eq:restate-purging}.
Assume that~$\mathbf{H}_{\ell}(q)=\sigma_{\ell}(q)\,\mathbf{r}_{\ell}(q)$ with
$\sigma_{\ell}(q)=\frac{\gamma_G(q) \overline{\beta}_{\ell}(q)}{\gamma_G(q)\overline{\beta}_{\ell}(q)+ \xi_{k}\beta_{\ell}(q)}$
holds for an integer~$\ell \geq0$.
Then, $\mathbf{H}_{\ell+1}(q)$ is computed via~\eqref{eq:successive-tree-update} as:
\begin{equation*}
\mathbf{H}_{\ell+1}(q)
=\frac{\|\sigma_{\ell}(q)\,\mathbf{r}_{\ell}(q)\| \,\mathbf{r}_{\ell+1}(q)}{\alpha_{\ell}(q) \, \|\mathbf{r}_{\ell}(q)\|+(1-\alpha_{\ell}(q))\,\|\sigma_{\ell}(q)\,\mathbf{r}_{\ell}(q)\|}.
\end{equation*}
Since $\sigma_{\ell}(q) \in [0,1]$,
it holds that $\mathbf{H}_{\ell+1}(q)
=\frac{\sigma_{\ell}(q) \,\mathbf{r}_{\ell+1}(q)}{\alpha_{\ell}(q) +(1-\alpha_{\ell}(q))\sigma_{\ell}(q)}$.
When combined with $\sigma_{\ell}(q) = \frac{\sigma_{\ell+1}(q)}{\alpha_{\ell}(q)\, (1-\sigma_{\ell+1}(q)) + \sigma_{\ell+1}(q)}$
from Lemma~\ref{lemma:update-obstacle-purging},
$\mathbf{H}_{\ell+1}(q)=\sigma_{\ell+1}(q) \, \mathbf{r}_{\ell+1}(q)$ holds,
which is consistent with~\eqref{eq:restate-purging}, thus concluding the proof.}
\end{proof}

\begin{remark}\label{remark:independent-stars}
{Note that the \emph{same} auxiliary function $\Psi(\cdot)$ in~\eqref{eq:general-iterative-function}
is used for the recursive computation
in \eqref{eq:workspace-update}, \eqref{eq:obstacle-update} and \eqref{eq:successive-tree-update}.}
\hfill $\blacksquare$
\end{remark}

\begin{remark}\label{remark:independent-stars}
{The parameters $\lambda_k$ in Def.~\ref{def:online-analytic-switch} for $k=0,1,\cdots, M$,
and $\xi_k$ in Def.~\ref{def:online-analytic-switch-purging} for $k=1,2,\cdots, N$
should satisfy the conditions in Lemma~\ref{lemma:point-world-nf},
i.e., $\lambda_k > \Lambda_k$ for some positive numbers $\Lambda_k >0$;
and $\xi_k > \Xi_k$ for some positive numbers $\Xi_k > 0$.}
\hfill $\blacksquare$
\end{remark}

\begin{figure*}[t]
  \centering
  \includegraphics[width=0.95\hsize]{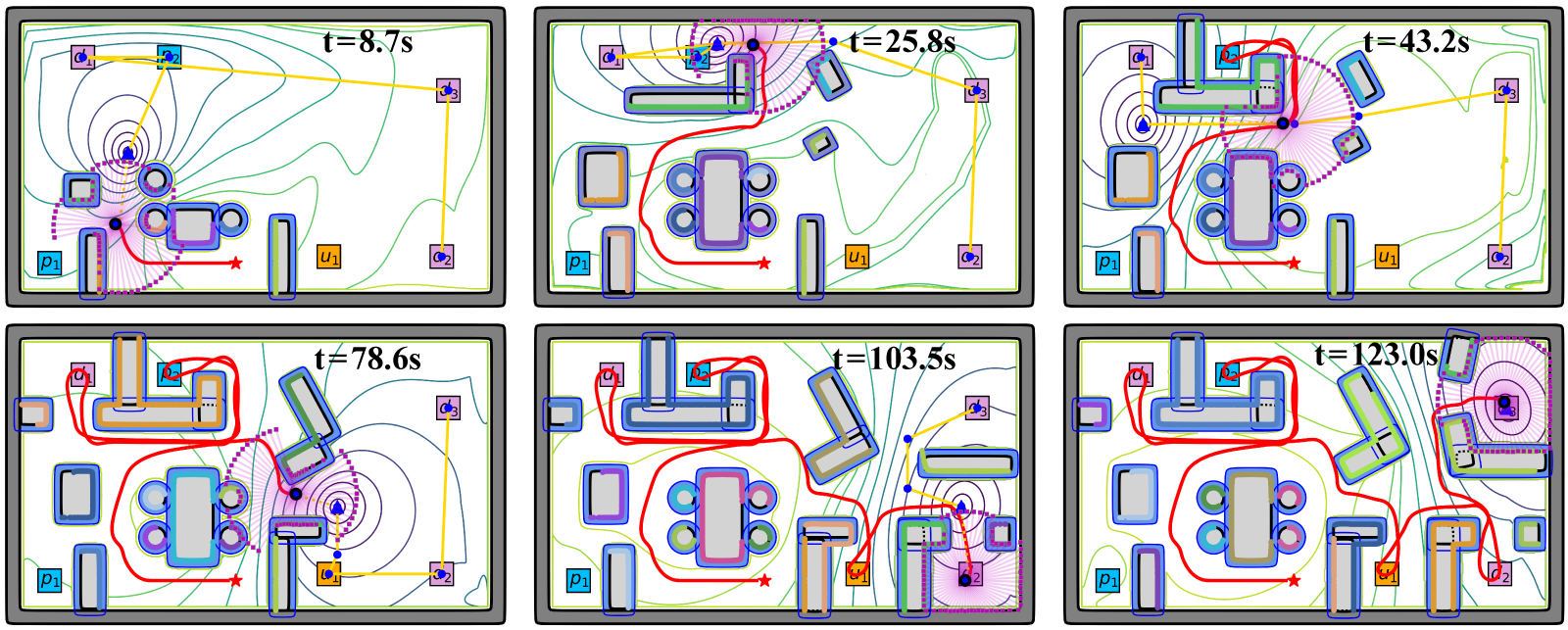}
  \vspace{-0.16in}
  \caption{Snapshots of simulation when the high-level plan and
    the underlying harmonic potentials are updated,
  as the robot explores gradually the workspace.
  The robot trajectory is depicted in red; the local goal of~$\varphi_{\text{NF}}$ are highlighted in blue triangle;
  and the intermediate waypoints in~$\mathcal{T}$ are marked in blue.}\label{fig:simulation}
  \vspace{-0.1in}
\end{figure*}

\textbf{Combination of both cases:}
Overall, the above two cases can be summarized into a more general formula.
Namely, the recursive transformation from forest of stars to sphere world
after adding the~$k$-th obstacle~$\mathcal{O}_k$ is given by:
\begin{equation}\label{eq:online-tf-f-m}
  \Phi_{\mathcal{F}\rightarrow \mathcal{M}}^{k}(q) \triangleq
\Phi_{\mathcal{S}\rightarrow \mathcal{M}}^{k} \circ
\Phi_{\mathcal{F}\rightarrow \mathcal{S}}^{k}(q),
\end{equation}
where if the new obstacle~$\mathcal{O}_{k+1}$ is added to the workspace as an independent star,
$\Phi_{\mathcal{S}\rightarrow \mathcal{M}}^{k}$ is updated to
$\Phi_{\mathcal{S}\rightarrow \mathcal{M}}^{k+1}$ by~\eqref{eq:online-tf-s-m};
on the other hand, if the new obstacle~$\mathcal{O}_{k+1}$ is overlapping with any existing obstacle,
$\Phi_{\mathcal{F}\rightarrow \mathcal{S}}^{k}$ is updated to
$\Phi_{\mathcal{F}\rightarrow \mathcal{S}}^{k+1}$ by~\eqref{eq:online-tf-f-s-update}.
Moreover, each time an independent star~$\mathcal{O}_{M+1}$ is added to the workspace,
the \emph{complete} harmonic potential function defined in point world should be updated by the recursive rule:
$$
 \phi_{\mathcal{P}}^{k+1}(x) = \frac{K}{K+1}\,\phi_{\mathcal{P}}^{k}(x) + \frac{1}{K+1}\big(\phi(x,\, P_d) + \phi(x, P_{M+1})\big),
$$
where~$\phi_{\mathcal{P}}^{k}(x)$ and~$\phi_{\mathcal{P}}^{k+1}(x)$ are the
potential function in the~$k$-th and~$(k+1)$-th step;~$\phi(x, P_{M+1})$
is the harmonic term for the new point obstacle.
Given the updated transformation~$\Phi_{\mathcal{F}\rightarrow \mathcal{M}}^{k+1}$ and the underlying
harmonic potentials in point world~$\phi_{\mathcal{P}}^{k+1}$,
the new navigation function is adapted to~$\varphi_\texttt{NF}^{k+1}$ accordingly.
\subsubsection{Local Revision of Harmonic Trees}
\label{subsubsec:online-tree}
As new obstacles are detected, parts of some Oriented Harmonic Trees
$\{\mathcal{T}_{\widehat{g}_{\ell}\rightarrow \widehat{g}_{\ell+1}}\}$
are revised in \emph{three} main steps:
(I) \emph{Trimming of Harmonic trees}.
For each HT $\mathcal{T}_{\widehat{g}_{\ell}\rightarrow \widehat{g}_{\ell+1}}$,
find the vertices that are within the newly-added obstacle,
i.e.,~$\nu\in \mathcal{O}_k$,
and remove them from the set of vertices~$V$, along with the attached edges;
(II) \emph{Potential-biased regeneration}.
To re-connect the oriented HT to the goal vertex,
new vertices and edges are generated.
Depending on the particular method of discretization,
new vertices could be generated with the bias,
e.g., more towards the area with large change in its potential value;
(III) \emph{Path revision}.
Once new vertices and edges are added,
the associated edge costs should be re-evaluated.
Suppose that at each time of update~$H$ edges have been traversed in total,
of which their actual costs are given by~$\boldsymbol{\gamma}^{\star}=(\gamma_1^{\star},\cdots, \gamma_H^{\star})$.
Their estimated costs are~$\Bar{\mathbf{\gamma}}=(\gamma_1, \cdots, \gamma_H)$
with $\Bar{\mathbf{d}}$ and $\Bar{\mathbf{Q}}$ being the distance and rotation cost from~\eqref{eq:nf-cost}.
Then, the weighting parameters can be updated by:
$\mathbf{w}^{\star} = (\Bar{\mathbf{Q}}^{\intercal}\Bar{\mathbf{Q}})^{-1}\Bar{\mathbf{Q}}^{\intercal}$ $(\boldsymbol{\gamma}^{\star}-\Bar{\mathbf{d}})$,
which is then used to update all edge costs, even for other HTs.
Afterwards, the revised path is found
via the same search algorithm to the goal vertex.

\subsubsection{Asymptotic Adaptation of Task Plan}
\label{subsubsec:online-automaton}
On the highest level,
since the HTs are updated, the associated feasibility and costs in the navigation
map~$\mathcal{G}$ should be updated accordingly,
which in turn leads to a different task plan~$\hat{\mathbf{g}}$.
More specifically,
the navigation map~$\mathcal{G}$ is updated in \emph{two} steps:
(i) the feasibility of navigation from~$\widehat{g}$ to~$\widehat{g}'$ in~$\mathcal{G}$
is re-evaluated based on whether the path~$\mathbf{P}_{\widehat{g}\rightarrow \widehat{g}'}$
exists within the HT $\mathcal{T}_{\widehat{g}\rightarrow \widehat{g}'}$.
If not, the transition $(g,\, g')$ is removed from the edge set~$E$.
This is often caused by newly-discovered obstacles blocking some passages;
(ii) the costs of transitions within~$E$ are re-computed
given the updated path within the new HTs.
Consequently, the navigation map~$\mathcal{G}$ is up-to-date with the latest environment model.
Thus, the task plan~$\widehat{\mathbf{g}}^\star$ is re-synthesized
within the updated product~$\hat{\mathcal{A}}$ by searching for a path
from the current product state to an accepting state.
In other words, as the environment is gradually explored, the suffix of the task plan often would converge to optimal one asymptotically.


\subsection{Complexity Analyses}\label{subsec:analysis}
The computational complexity to construct the product
$\widehat{\mathcal{A}}$ is $\mathcal{O}(|G|^2 |\mathcal{A}_{\varphi}|^2)$,
where $|G|$ is the number of goal regions and $|\mathcal{A}_{\varphi}|$ is the number of states within $\mathcal{A}_{\varphi}$.
The complexity to construct a HT $\mathcal{T}$ is $\mathcal{O}(|\mathcal{T}|)$, where $|\mathcal{T}|$
is the number of intermediate waypoints in the tree structure.
For a forest world with~$|\mathcal{I}|$ total stars and $|\mathcal{F}|$ trees of stars with maximum depth $d$,
the complexity to compute the initial Harmonic potential $\Upsilon$ is $\mathcal{O}(|\mathcal{F}|^2+d|\mathcal{I}|)$.
Thus, the complexity to compute the complete $\mathcal{G}$ is $\mathcal{O}(|G|^2 |\mathcal{A}_{\varphi}|^2+|\mathcal{T}|(|\mathcal{F}|^2+d|\mathcal{I}|))$.
During online adaptation, each time an additional obstacle is added,
the recursive update of $\Upsilon$ has complexity $\mathcal{O}(|\mathcal{F}|+|\mathcal{I}|)$.
The complexity of local revision of~$\mathcal{G}$ and $\mathcal{T}$ are $\mathcal{O}(|\mathcal{G}'| (|\mathcal{F}|+|\mathcal{I}|))$ and $\mathcal{O}(|\mathcal{T}'|)$,
respectively, where $|\mathcal{G}'|$ is the number of revised vertices,
$|\mathcal{T}'|$ is the number of generated intermediate waypoints.



\section{Simulation and Experiments} \label{sec:experiments}
To further validate the effectiveness of our proposed method,
extensive numerical simulations and hardware experiments are conducted,
against several state-of-the-art approaches.
The proposed method is implemented in Python3 and tested on an Intel Core i7-1280P CPU.
More descriptions, accompanied videos,
and source code can be found in the website~\cite{wang2023harmonic}.

\subsection{Description of Robot and Task}\label{subsec:model}
Consider a unicycle robot with the dynamics in~\eqref{eq:unicycle}
that operates within a workspace measuring
$7m \times 4m$, with an initial model depicted in Fig.~\ref{fig:problem}.
The robot occupies a circular area of radius $r_{\texttt{r}}=0.1m$,
and is equipped with a Lidar sensor capable of detecting objects
up to a maximum range of $1.0m$.
The control gains as specified in~\eqref{eq:control}
are set to~$k_\upsilon = 1.0$ and~$k_\omega = 0.8$, such that the tracking accuracy is below $0.1m$.
Both the robot control and the Lidar measurements are updated at $10$Hz.
The design parameter for the smooth switch function in~\eqref{eq:theta1} and~\eqref{eq:theta2} is set to~$0.5$.
The parameter~$K$ in the harmonic potentials
  defined in~\eqref{eq:harmonic-point-potential} is set to $2$ initially.
Each time an independent obstacle is added to the workspace,
$K$ is increased by~$1$.
In addition, the parameter~$\mu$ is set to $1$,
while the geometric parameters~$E_i$ and~$E_G$ are both set to~$0.02$.
The parameters~$\lambda_k$,~$\xi_k$ in
Def.~\ref{def:online-analytic-switch} and~\ref{def:online-analytic-switch-purging}
are set to $5.0\times10^2$ and $1.0\times10^5$,
i.e., sufficiently large to accommodate the complex environment.
The weight parameter~$\mathbf{w}$ utilized in the control cost estimation in~\eqref{eq:nf-cost}
is initially given as~$[0.1,\ 0.1]^{\intercal}$.
The set of vertices~$V$ in~\eqref{eq:tree} is generated by the visibility graph from~\cite{huang2004dynamic}
with cost estimation based on~\eqref{eq:nf-cost} and a safety buffer of $0.15m$.
Fig.~\ref{fig:problem} illustrates the complete environment,
which mimics a complex office setting with numerous overlapping obstacles.
Initially, the robot \emph{only knows} the workspace boundary, the task regions and none of the obstacles.

Regarding the robot task, the workspace consists of a set of regions of interest,
denoted by~$p_1, p_2, d_1, d_2, d_3, u_1$.
The high-level task requires the robot
to transport objects from either the storage room $p_1$ or $p_2$
to the destinations $d_1$, $d_2$, and $d_3$.
Moreover, during runtime, a contingent task might be triggered by an external event
and requested in region $u_1$.
In such cases, the robot must prioritize this task for execution.
This can be expressed as the formula $\varphi = \big(\Diamond \big{(} (p_1 \vee p_2)
\wedge (\Diamond d_1)\big{)}
\wedge \big(\Diamond \big{(} (p_1 \vee p_2) \wedge (\Diamond d_2)\big{)}
\wedge \big(\Diamond \big{(} (p_1 \vee p_2) \wedge (\Diamond d_3)\big{)}
\wedge (\texttt{o}\rightarrow u_5)$.
It is worth noting that there exist numerous high-level plans to fulfill the specified tasks,
and their actual costs rely heavily on the complete workspace,
which can only be explored online.

\subsection{Results}\label{subsec:results}
As visualized in Fig.~\ref{fig:simulation}, the robot starts from the initial position
$(3.2m, 0.4m)$ with the orientation~$\pi$.
The task automaton $\mathcal{A}_{\varphi}$ is constructed with $29$ states and $474$ edges.
The FTS associated with the regions of interest is initialized as a fully-connected graph
that each edge is built as a HT consisting of $4$ intermediate waypoints,
with an average computation time of $0.08s$.
Then, the product automaton $\widehat{\mathcal{A}}$ is constructed with $192$ states and $1435$ edges,
of which the initial task plan is~$P_0 = p_1d_1d_3d_2$.
Guided by this high-level plan, the robot navigates first towards task $p_1$ along the intermediate waypoints.
Between any pair of the waypoint $(\nu_{n},\,\nu_{n+1})$,
the initial harmonic potential~$\varphi_{\texttt{NF}}$ is constructed by~\eqref{eq:complete-nf} with an average
computation time of~$10.46ms$.
Then, the oriented harmonic fields~$\Upsilon$ is constructed with almost no additional time,
given the desired orientation~$\theta_{n+1}$.
The switch function in~\eqref{eq:switch-control} is utilized
with~$k_s=2.0$ and~$\epsilon=0.1$.
Furthermore, the underlying HT and high-level plan are updated at discrete instants during execution, e.g., at $t=8.7s$ and $43.2s$
when new obstacles are detected and estimated, blocking the way to the next task region.
In particular, at $t=8.7s$, the estimated costs for the plans $p_1d_1d_3d_2$
and $p_2d_1d_3d_2$ are $22.91$ and $16.07$, respectively.
Thus, the robot moves toward~$p_2$ instead of~$p_1$ due to the higher cost associated with rotation.
The mean computation time for the new navigation functions at $t=8.7s$ and $43.2s$ are $14.06ms$ and $15.82ms$, respectively.
The weight~$\mathbf{w}$ in the control cost is updated to $[0.67, 0.34]^{\intercal}$ given the traversed edges in HT.
Afterwards, the navigation map~$\mathcal{G}$ is updated in~$0.16s$ given the new visibility graph.
At $t=78.6s$, the urgent task in region $u_1$ is triggered, thus the new subtask is added and
the associated product automaton is reconstructed to get the new task plan $u_1 d_2 d_3$,
as shown in Fig.~\ref{fig:simulation}.
Finally, the whole task is accomplished in~$123s$ and the resulting trajectory
has a total distance~$27.29m$.
It is worth pointing out that the final estimated forest world is \emph{not the
same} as the actual workspace in Fig.~\ref{fig:problem},
e.g., the L-shape obstacle to the middle-left,
and the squircle to the bottom-right are not fully explored.

\begin{figure}[t]
  \centering
  \includegraphics[width=0.99\hsize]{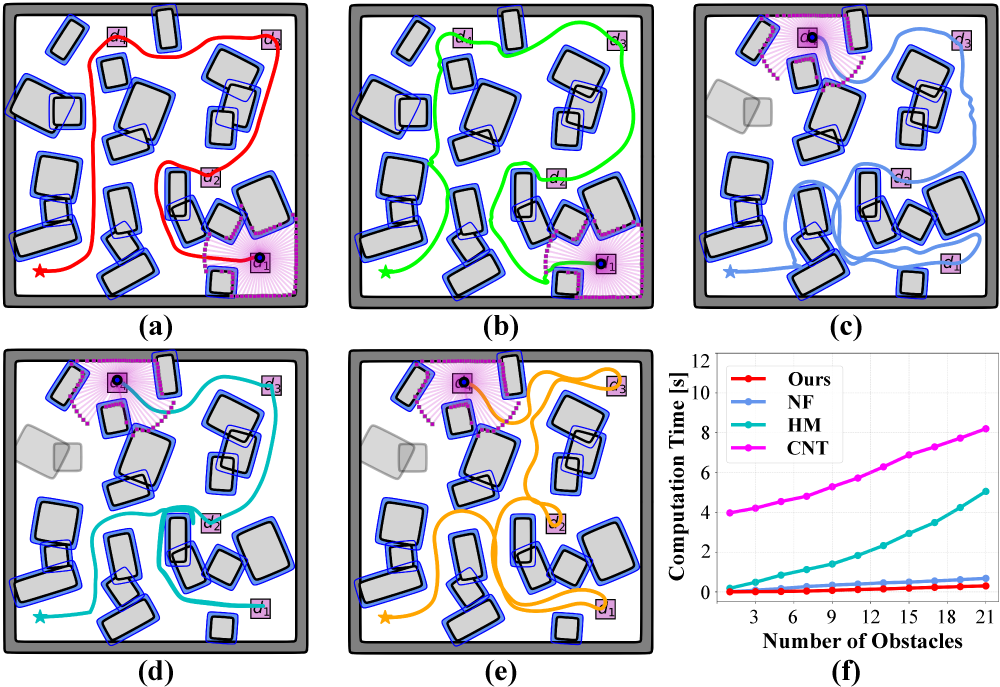}
  \vspace{-0.1in}
  \caption{(\textbf{a}-\textbf{e})
      Final trajectories via the proposed method (in red),
      OHPF without the search trees (in green),
      NF without plan adaptation (in blue),
     HM without plan adaptation (in cyan),
     and RRT$^\star$ without plan adaptation (in orange).
     The obstacles marked in blue are detected online, while those in
     light gray are undetected;
     (\textbf{f}) Comparison of computation time between our
     incremental method (in red) and the baselines.}
    \label{fig:compare}
\end{figure}

\subsection{Comparisons}\label{subsec:compare}
\begin{table}[t]
\begin{center}
  \caption{Comparison with Baselines}\label{table:table-data}
  \vspace{-0.05in}
  \setlength{\tabcolsep}{0.5\tabcolsep}
  \centering
  \begin{tabular}{c||c c c c c}
    \toprule
    \textbf{Method} & \textbf{\makecell{Planning \\ Time [s]}} & \textbf{\makecell{Execution \\ Time [s]}} & \textbf{\makecell{Travel Di- \\ stance [m]}} & \textbf{\makecell{Adap- \\ tation}} & \textbf{\makecell{Osci- \\ llation}}\\
    \midrule
    \textbf{HT(ours)} & {4.28} & \textbf{47.64}  & \textbf{15.32} & \textbf{Yes} & \textbf{No}\\
    OHPF        & 3.19    & 65.80        & 18.75          & \textbf{Yes}          & Yes\\
    NF          & \textbf{0.0}   & 85.23         & 22.16          & No           & Yes\\
    HM       & 13.09   & 54.29   & 18.85   & {No}   & {\textbf{No}}\\
    RRT$^\star$ & 97.68   & 69.81        & 20.50          & No           & \textbf{No}\\
    CNT        & 47.26   & 77.27    & 19.64   & {No}   & {\textbf{No}}\\
    \bottomrule
  \end{tabular}
  \vspace{-0.1in}
\end{center}
\end{table}
To further demonstrate the effectiveness of our proposed Harmonic Tree (HT) structure,
{a quantitative comparison is conducted against \textbf{five} baselines:
(i) the oriented harmonic potential fields
(OHPF) proposed in this work with the high-level task adaptation,
but omitting the second-layer search trees;
(ii) the Navigation Function (NF) from~\cite{rimon1990exact,loizou2017navigation}
without the high-level task adaptation;
(iii) the harmonic maps (HM) in~\cite{vlantis2018robot} via the open-sourced implementation;
(iv) the Non-holonomic RRT$^\star$ in~\cite{lavalle2006planning, park2015feedback}
without the high-level task adaptation;
and (v) the conformal navigation transformation (CNT) from~\cite{fan2022robot},
which has to be modified considerably for the complex workspace here.}
As summarized in Table~\ref{table:table-data},
the metrics to compare are computation time for each plan synthesis and adaptation,
the total cost of final trajectory as the distance travelled,
and whether oscillations appear during execution.

For a more detailed comparison,
an obstacle cluttered work-
space of size~$5.0m \times 5.0m$ with~$20$ overlapping obstacles
is considered as shown in Fig.~\ref{fig:compare}.
It is worth noting that the associated forest world has a maximum depth of~$4$,
which is quite difficult for model-based methods such as NF, OHPF and ours
due to the purging process.
The robot starts from $(0.4m, 0.4m)$ and has a sensing range of~$1.0m$.
The task requires the robot to surveil the regions of interest
$d_1, d_2, d_3, d_4$ in an arbitrary order.
Other parameters are same to Sec.~\ref{subsec:model}.
The final trajectories and numerical results are shown in Fig.~\ref{fig:compare}
and Table~\ref{table:table-data}.
The proposed HT exhibits lowest cost for task completion with a fast planning and adaptation,
with no collision or oscillations during execution.
As for NF, a blind execution of the initial plan leads to a more costly trajectory with numerous oscillations since the gradients near the obstalce can not be tracked perfectly.
Besides, the nominal NF can not control the final orientation as the robot approaches the task areas,
resulting in high steering cost when task regions switch.
In contrast, OHPF optimizes the orientations at each task region and adapts the high-level
plan online as the proposed HT,
leading to a much smaller execution time and overall cost.
However, it is worth noting that without the guidance of search trees,
oscillations can still occur close to obstacles for OHPF.
Further, RRT$^\star$ takes significantly more time to synthesize the initial plan
and adapt the new plan,
due to high computationally complexity of collision detection between samples.
Namely, it takes around $98s$ to compute the complete plan for RRT$^\star$, compared with merely $4.3s$ by the proposed HT.
The HM method has a travel distance of~$18.85m$ with a total planning time~$13.09s$,
  which includes $20$ times of replanning and each replanning takes~$0.65s$.
  Each obstacle is represented by average~$80$ boundary points, which is the minimum
  number that can ensure safety in our tests.
Similarly, the CNT method takes significantly more planning time and execution
time (close to $26s$ and $77s$, with~$20$ times of replanning).

Lastly, the computational efficiency of the proposed iterative approach
is compared with nominal NF, HM, and CNT, as summarized in Fig.~\ref{fig:compare}.
As the number of obstacles increases, the computation time of our method remains relatively low
due to its analytical computation, incremental update and the hierarchical structure.
In contrast, the nominal NF method requires nearly twice the time due to the
frequent re-calculation of the potentials from scratch.
Furthermore, the computation time for CNT increases drastically
with the number of boundary points (each obstacle has $150$ boundary points),
e.g., each replanning takes more than~$8.0s$ for CNT with~$20$ obstacles.

\subsection{Hardware Experiments}\label{subsec:experiments}
The proposed method is deployed to a differential-driven robot of radius $r_{\texttt{r}}=0.2m$.
As shown in Fig.~\ref{fig:hardware}, the workspace is constructed
  with dimensions of $5.2m \times 5.2m$,
  which has $4$ independent and $4$ overlapping obstacles.
The controller gains in~\eqref{eq:control} are set to~$k_\upsilon = 0.1$
  and~$k_\omega = 0.2$, ensuring that the tracking error remains below $0.2m$.
The robot state is estimated using SLAM technology, while the communication
  between the robot and the workstation is achieved by ROS with a frequency of $10$Hz.
The point clouds are collected by a forward-facing $180^\circ$ Lidar within a radius of~$8.0$m around the robot.
Other parameters are similar to the simulation.
The task is designed to transport objects from $p_1$ to $d_1$, $d_2$, and $d_3$.
The robot starts from its initial position at $(0.0m, 0.0m)$.
The initial task plan is derived within $0.09s$ as $p_1d_1d_2d_3$.
The robot adjusts its task plan and
harmonic potentials as $6$ obstacles are detected online.
As the robot detects more obstacles and familiar with the workspace, the task plan is updated as
$p_1 d_2 d_3 d_1$ and finally converges to $p_1d_2d_3d_1$.
The whole task is accomplished in $338s$, of which the complete video can be found in~\cite{wang2023harmonic}
The resulting trajectory and snapshots are shown in Fig.~\ref{fig:hardware} with a total distance~$21.74m$,
which is safe and smooth despite of the localization and control uncertainties.
These results are consistent with the simulations.

\begin{figure}[t]
  \centering
  \includegraphics[width=0.99\hsize]{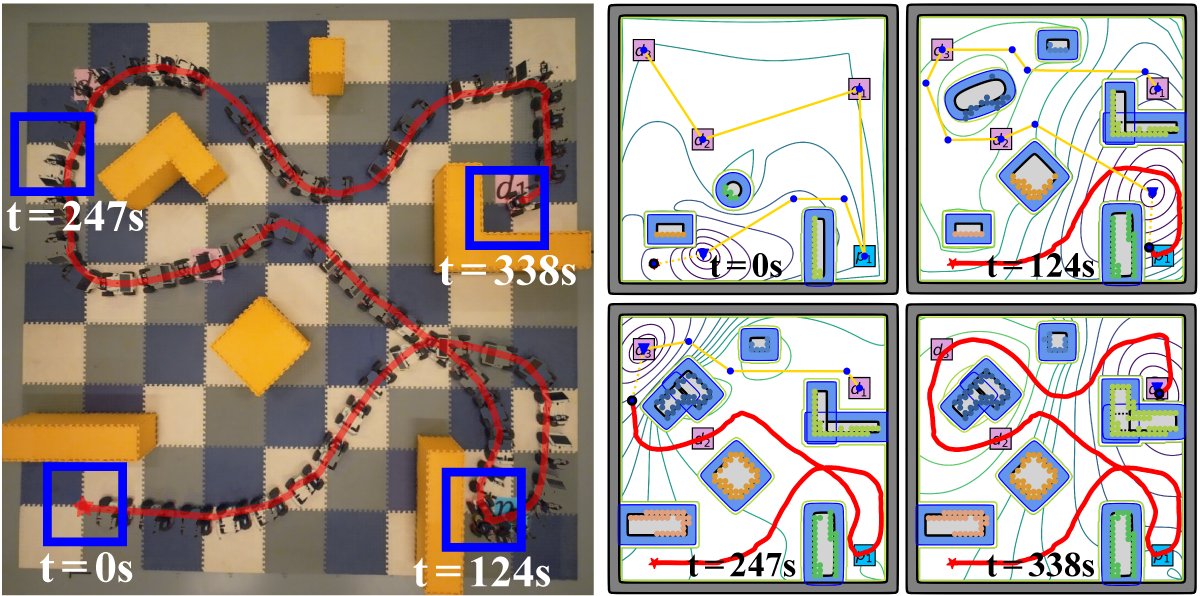}
  \vspace{-0.1in}
  \caption{\textbf{Left:} Recorded execution results;
    \textbf{Right:} Snapshots of robot trajectories and the potential fields,
    where obstacles in blue are detected online.}\label{fig:hardware}
    \vspace{-0.05in}
\end{figure}

\section{Conclusion} \label{sec:conclusion}
This work proposes an automated framework for task and motion planning,
employing harmonic potentials for navigation and oriented search trees for planning.
The design and construction of the search tree is specifically customized for the task automaton
and co-designed with the underlying navigation controllers based on harmonic potentials.
Efficient and secure task execution is ensured for partially-known workspace.
As described earlier, although an online approach is proposed
to adapt the forest world and obstacle estimation during execution,
it lacks a systematic analysis for more general workspaces,
which is part of our ongoing work.
Moreover, the extension to 3D navigation remains our future work.

\bibliography{contents/references}

\begin{thebibliography}{53}
\expandafter\ifx\csname natexlab\endcsname\relax\def\natexlab#1{#1}\fi
\providecommand{\url}[1]{\texttt{#1}}
\providecommand{\href}[2]{#2}
\providecommand{\path}[1]{#1}
\providecommand{\DOIprefix}{doi:}
\providecommand{\ArXivprefix}{arXiv:}
\providecommand{\URLprefix}{URL: }
\providecommand{\Pubmedprefix}{pmid:}
\providecommand{\doi}[1]{\href{http://dx.doi.org/#1}{\path{#1}}}
\providecommand{\Pubmed}[1]{\href{pmid:#1}{\path{#1}}}
\providecommand{\bibinfo}[2]{#2}
\ifx\xfnm\relax \def\xfnm[#1]{\unskip,\space#1}\fi
\bibitem[{Baier \& Katoen(2008)}]{baier2008principles}
\bibinfo{author}{Baier, C.}, \& \bibinfo{author}{Katoen, J.-P.}
  (\bibinfo{year}{2008}).
\newblock {\it \bibinfo{title}{Principles of model checking}\/}.
\newblock \bibinfo{publisher}{MIT press}.
\bibitem[{Dahlin \& Karayiannidis(2023{\natexlab{a}})}]{dahlin2023creating}
\bibinfo{author}{Dahlin, A.}, \& \bibinfo{author}{Karayiannidis, Y.}
  (\bibinfo{year}{2023}{\natexlab{a}}).
\newblock \bibinfo{title}{Creating star worlds: Reshaping the robot workspace
  for online motion planning}.
\newblock {\it \bibinfo{journal}{IEEE Transactions on Robotics}\/},  {\it
  \bibinfo{volume}{39}\/}, \bibinfo{pages}{3655--3670}.
\bibitem[{Dahlin \& Karayiannidis(2023{\natexlab{b}})}]{dahlin2023obstacle}
\bibinfo{author}{Dahlin, A.}, \& \bibinfo{author}{Karayiannidis, Y.}
  (\bibinfo{year}{2023}{\natexlab{b}}).
\newblock \bibinfo{title}{Obstacle avoidance in dynamic environments via
  tunnel-following mpc with adaptive guiding vector fields}.
\newblock In {\it \bibinfo{booktitle}{IEEE Conference on Decision and Control
  (CDC)}\/} (pp. \bibinfo{pages}{5784--5789}).
\bibitem[{Fainekos et~al.(2009)Fainekos, Girard, Kress-Gazit \&
  Pappas}]{fainekos2009temporal}
\bibinfo{author}{Fainekos, G.~E.}, \bibinfo{author}{Girard, A.},
  \bibinfo{author}{Kress-Gazit, H.}, \& \bibinfo{author}{Pappas, G.~J.}
  (\bibinfo{year}{2009}).
\newblock \bibinfo{title}{Temporal logic motion planning for dynamic robots}.
\newblock {\it \bibinfo{journal}{Automatica}\/},  {\it \bibinfo{volume}{45}\/},
  \bibinfo{pages}{343--352}.
\bibitem[{Fan et~al.(2022)Fan, Liu, Zhang \& Xu}]{fan2022robot}
\bibinfo{author}{Fan, L.}, \bibinfo{author}{Liu, J.}, \bibinfo{author}{Zhang,
  W.}, \& \bibinfo{author}{Xu, P.} (\bibinfo{year}{2022}).
\newblock \bibinfo{title}{Robot navigation in complex workspaces using
  conformal navigation transformations}.
\newblock {\it \bibinfo{journal}{IEEE Robotics and Automation Letters}\/},
  {\it \bibinfo{volume}{8}\/}, \bibinfo{pages}{192--199}.
\bibitem[{Faulwasser \& Findeisen(2015)}]{faulwasser2015nonlinear}
\bibinfo{author}{Faulwasser, T.}, \& \bibinfo{author}{Findeisen, R.}
  (\bibinfo{year}{2015}).
\newblock \bibinfo{title}{Nonlinear model predictive control for constrained
  output path following}.
\newblock {\it \bibinfo{journal}{IEEE Transactions on Automatic Control}\/},
  {\it \bibinfo{volume}{61}\/}, \bibinfo{pages}{1026--1039}.
\bibitem[{Filippidis \& Kyriakopoulos(2011)}]{filippidis2011adjustable}
\bibinfo{author}{Filippidis, I.}, \& \bibinfo{author}{Kyriakopoulos, K.~J.}
  (\bibinfo{year}{2011}).
\newblock \bibinfo{title}{Adjustable navigation functions for unknown sphere
  worlds}.
\newblock In {\it \bibinfo{booktitle}{IEEE Conference on Decision and Control
  and European Control Conference}\/} (pp. \bibinfo{pages}{4276--4281}).
\bibitem[{Garrett et~al.(2021)Garrett, Chitnis, Holladay, Kim, Silver,
  Kaelbling \& Lozano-P{\'e}rez}]{garrett2021integrated}
\bibinfo{author}{Garrett, C.~R.}, \bibinfo{author}{Chitnis, R.},
  \bibinfo{author}{Holladay, R.}, \bibinfo{author}{Kim, B.},
  \bibinfo{author}{Silver, T.}, \bibinfo{author}{Kaelbling, L.~P.}, \&
  \bibinfo{author}{Lozano-P{\'e}rez, T.} (\bibinfo{year}{2021}).
\newblock \bibinfo{title}{Integrated task and motion planning}.
\newblock {\it \bibinfo{journal}{Annual review of control, robotics, and
  autonomous systems}\/},  {\it \bibinfo{volume}{4}\/},
  \bibinfo{pages}{265--293}.
\bibitem[{Gavin(2019)}]{gavin2019levenberg}
\bibinfo{author}{Gavin, H.~P.} (\bibinfo{year}{2019}).
\newblock \bibinfo{title}{The levenberg-marquardt algorithm for nonlinear least
  squares curve-fitting problems}.
\newblock {\it \bibinfo{journal}{Department of civil and environmental
  engineering, Duke University}\/},  {\it \bibinfo{volume}{19}\/}.
\bibitem[{Ghallab et~al.(2004)Ghallab, Nau \& Traverso}]{ghallab2004automated}
\bibinfo{author}{Ghallab, M.}, \bibinfo{author}{Nau, D.}, \&
  \bibinfo{author}{Traverso, P.} (\bibinfo{year}{2004}).
\newblock {\it \bibinfo{title}{Automated Planning: theory and practice}\/}.
\newblock \bibinfo{publisher}{Elsevier}.
\bibitem[{Guo et~al.(2018)Guo, Andersson \& Dimarogonas}]{guo2018human}
\bibinfo{author}{Guo, M.}, \bibinfo{author}{Andersson, S.}, \&
  \bibinfo{author}{Dimarogonas, D.~V.} (\bibinfo{year}{2018}).
\newblock \bibinfo{title}{Human-in-the-loop mixed-initiative control under
  temporal tasks}.
\newblock In {\it \bibinfo{booktitle}{IEEE International Conference on Robotics
  and Automation (ICRA)}\/} (pp. \bibinfo{pages}{6395--6400}).
\bibitem[{Guo \& Dimarogonas(2015)}]{guo2015multi}
\bibinfo{author}{Guo, M.}, \& \bibinfo{author}{Dimarogonas, D.~V.}
  (\bibinfo{year}{2015}).
\newblock \bibinfo{title}{Multi-agent plan reconfiguration under local ltl
  specifications}.
\newblock {\it \bibinfo{journal}{The International Journal of Robotics
  Research}\/},  {\it \bibinfo{volume}{34}\/}, \bibinfo{pages}{218--235}.
\bibitem[{Hsu et~al.(2006)Hsu, Latombe \& Kurniawati}]{hsu2006probabilistic}
\bibinfo{author}{Hsu, D.}, \bibinfo{author}{Latombe, J.-C.}, \&
  \bibinfo{author}{Kurniawati, H.} (\bibinfo{year}{2006}).
\newblock \bibinfo{title}{On the probabilistic foundations of probabilistic
  roadmap planning}.
\newblock {\it \bibinfo{journal}{The International Journal of Robotics
  Research}\/},  {\it \bibinfo{volume}{25}\/}, \bibinfo{pages}{627--643}.
\bibitem[{Huang \& Chung(2004)}]{huang2004dynamic}
\bibinfo{author}{Huang, H.-P.}, \& \bibinfo{author}{Chung, S.-Y.}
  (\bibinfo{year}{2004}).
\newblock \bibinfo{title}{Dynamic visibility graph for path planning}.
\newblock In {\it \bibinfo{booktitle}{IEEE/RSJ International Conference on
  Intelligent Robots and Systems (IROS)}\/} (pp. \bibinfo{pages}{2813--2818}).
\newblock volume~\bibinfo{volume}{3}.
\bibitem[{Huber et~al.(2019)Huber, Billard \& Slotine}]{huber2019avoidance}
\bibinfo{author}{Huber, L.}, \bibinfo{author}{Billard, A.}, \&
  \bibinfo{author}{Slotine, J.-J.} (\bibinfo{year}{2019}).
\newblock \bibinfo{title}{Avoidance of convex and concave obstacles with
  convergence ensured through contraction}.
\newblock {\it \bibinfo{journal}{IEEE Robotics and Automation Letters}\/},
  {\it \bibinfo{volume}{4}\/}, \bibinfo{pages}{1462--1469}.
\bibitem[{Huber et~al.(2022)Huber, Slotine \& Billard}]{huber2022avoiding}
\bibinfo{author}{Huber, L.}, \bibinfo{author}{Slotine, J.-J.}, \&
  \bibinfo{author}{Billard, A.} (\bibinfo{year}{2022}).
\newblock \bibinfo{title}{Avoiding dense and dynamic obstacles in enclosed
  spaces: Application to moving in crowds}.
\newblock {\it \bibinfo{journal}{IEEE Transactions on Robotics}\/},  {\it
  \bibinfo{volume}{38}\/}, \bibinfo{pages}{3113--3132}.
\bibitem[{Huber et~al.(2024)Huber, Slotine \& Billard}]{huber2024avoidance}
\bibinfo{author}{Huber, L.}, \bibinfo{author}{Slotine, J.-J.}, \&
  \bibinfo{author}{Billard, A.} (\bibinfo{year}{2024}).
\newblock \bibinfo{title}{Avoidance of concave obstacles through rotation of
  nonlinear dynamics}.
\newblock {\it \bibinfo{journal}{IEEE Transactions on Robotics}\/},  {\it
  \bibinfo{volume}{40}\/}, \bibinfo{pages}{1983--2002}.
\bibitem[{Janson et~al.(2015)Janson, Schmerling, Clark \&
  Pavone}]{janson2015fast}
\bibinfo{author}{Janson, L.}, \bibinfo{author}{Schmerling, E.},
  \bibinfo{author}{Clark, A.}, \& \bibinfo{author}{Pavone, M.}
  (\bibinfo{year}{2015}).
\newblock \bibinfo{title}{Fast marching tree: A fast marching sampling-based
  method for optimal motion planning in many dimensions}.
\newblock {\it \bibinfo{journal}{The International Journal of Robotics
  Research}\/},  {\it \bibinfo{volume}{34}\/}, \bibinfo{pages}{883--921}.
\bibitem[{Karaman \& Frazzoli(2011)}]{karaman2011sampling}
\bibinfo{author}{Karaman, S.}, \& \bibinfo{author}{Frazzoli, E.}
  (\bibinfo{year}{2011}).
\newblock \bibinfo{title}{Sampling-based algorithms for optimal motion
  planning}.
\newblock {\it \bibinfo{journal}{The International Journal of Robotics
  Research}\/},  {\it \bibinfo{volume}{30}\/}, \bibinfo{pages}{846--894}.
\bibitem[{Khalil(2002)}]{khalil2002nonlinear}
\bibinfo{author}{Khalil, H.} (\bibinfo{year}{2002}).
\newblock {\it \bibinfo{title}{Nonlinear systems; 3rd ed.}\/}.
\newblock \bibinfo{address}{Upper Saddle River, NJ}:
  \bibinfo{publisher}{Prentice Hall}.
\bibitem[{Khatib(1986)}]{khatib1986real}
\bibinfo{author}{Khatib, O.} (\bibinfo{year}{1986}).
\newblock \bibinfo{title}{Real-time obstacle avoidance for manipulators and
  mobile robots}.
\newblock In {\it \bibinfo{booktitle}{Autonomous robot vehicles}\/} (pp.
  \bibinfo{pages}{396--404}).
\newblock \bibinfo{publisher}{Springer}.
\bibitem[{Khatib(1999)}]{khatib1999mobile}
\bibinfo{author}{Khatib, O.} (\bibinfo{year}{1999}).
\newblock \bibinfo{title}{Mobile manipulation: The robotic assistant}.
\newblock {\it \bibinfo{journal}{Robotics and Autonomous Systems}\/},  {\it
  \bibinfo{volume}{26}\/}, \bibinfo{pages}{175--183}.
\bibitem[{Kim et~al.(2022)Kim, Shimanuki, Kaelbling \&
  Lozano-P{\'e}rez}]{kim2022representation}
\bibinfo{author}{Kim, B.}, \bibinfo{author}{Shimanuki, L.},
  \bibinfo{author}{Kaelbling, L.~P.}, \& \bibinfo{author}{Lozano-P{\'e}rez, T.}
  (\bibinfo{year}{2022}).
\newblock \bibinfo{title}{Representation, learning, and planning algorithms for
  geometric task and motion planning}.
\newblock {\it \bibinfo{journal}{The International Journal of Robotics
  Research}\/},  {\it \bibinfo{volume}{41}\/}, \bibinfo{pages}{210--231}.
\bibitem[{Kim \& Khosla(1992)}]{kim1992real}
\bibinfo{author}{Kim, J.-o.}, \& \bibinfo{author}{Khosla, P.}
  (\bibinfo{year}{1992}).
\newblock \bibinfo{title}{Real-time obstacle avoidance using harmonic potential
  functions}.
\newblock {\it \bibinfo{journal}{IEEE Transactions on Robotics and
  Automation}\/},  {\it \bibinfo{volume}{8}\/}, \bibinfo{pages}{338--349}.
\bibitem[{Ko et~al.(2013)Ko, Kim \& Park}]{ko2013vf}
\bibinfo{author}{Ko, I.}, \bibinfo{author}{Kim, B.}, \& \bibinfo{author}{Park,
  F.~C.} (\bibinfo{year}{2013}).
\newblock \bibinfo{title}{Vf-rrt: Introducing optimization into randomized
  motion planning}.
\newblock In {\it \bibinfo{booktitle}{IEEE Asian Control Conference (ASCC)}\/}
  (pp. \bibinfo{pages}{1--5}).
\bibitem[{Koditschek(1987)}]{koditschek1987exact}
\bibinfo{author}{Koditschek, D.} (\bibinfo{year}{1987}).
\newblock \bibinfo{title}{Exact robot navigation by means of potential
  functions: Some topological considerations}.
\newblock In {\it \bibinfo{booktitle}{IEEE International Conference on Robotics
  and Automation (ICRA)}\/} (pp. \bibinfo{pages}{1--6}).
\newblock volume~\bibinfo{volume}{4}.
\bibitem[{LaValle(2006)}]{lavalle2006planning}
\bibinfo{author}{LaValle, S.~M.} (\bibinfo{year}{2006}).
\newblock {\it \bibinfo{title}{Planning algorithms}\/}.
\newblock \bibinfo{publisher}{Cambridge university press}.
\bibitem[{Leahy et~al.(2021)Leahy, Serlin, Vasile, Schoer, Jones, Tron \&
  Belta}]{leahy2021scalable}
\bibinfo{author}{Leahy, K.}, \bibinfo{author}{Serlin, Z.},
  \bibinfo{author}{Vasile, C.-I.}, \bibinfo{author}{Schoer, A.},
  \bibinfo{author}{Jones, A.~M.}, \bibinfo{author}{Tron, R.}, \&
  \bibinfo{author}{Belta, C.} (\bibinfo{year}{2021}).
\newblock \bibinfo{title}{Scalable and robust algorithms for task-based
  coordination from high-level specifications (scratches)}.
\newblock {\it \bibinfo{journal}{IEEE Transactions on Robotics}\/},  {\it
  \bibinfo{volume}{38}\/}, \bibinfo{pages}{2516--2535}.
\bibitem[{Li \& Tanner(2018)}]{li2018navigation}
\bibinfo{author}{Li, C.}, \& \bibinfo{author}{Tanner, H.~G.}
  (\bibinfo{year}{2018}).
\newblock \bibinfo{title}{Navigation functions with time-varying destination
  manifolds in star worlds}.
\newblock {\it \bibinfo{journal}{IEEE Transactions on Robotics}\/},  {\it
  \bibinfo{volume}{35}\/}, \bibinfo{pages}{35--48}.
\bibitem[{Lindemann et~al.(2021)Lindemann, Matni \& Pappas}]{lindemann2021stl}
\bibinfo{author}{Lindemann, L.}, \bibinfo{author}{Matni, N.}, \&
  \bibinfo{author}{Pappas, G.~J.} (\bibinfo{year}{2021}).
\newblock \bibinfo{title}{Stl robustness risk over discrete-time stochastic
  processes}.
\newblock In {\it \bibinfo{booktitle}{IEEE Conference on Decision and Control
  (CDC)}\/} (pp. \bibinfo{pages}{1329--1335}).
\bibitem[{Loizou(2011)}]{loizou2011closed}
\bibinfo{author}{Loizou, S.~G.} (\bibinfo{year}{2011}).
\newblock \bibinfo{title}{Closed form navigation functions based on harmonic
  potentials}.
\newblock In {\it \bibinfo{booktitle}{IEEE Conference on Decision and Control
  and European Control Conference}\/} (pp. \bibinfo{pages}{6361--6366}).
\bibitem[{Loizou(2017)}]{loizou2017navigation}
\bibinfo{author}{Loizou, S.~G.} (\bibinfo{year}{2017}).
\newblock \bibinfo{title}{The navigation transformation}.
\newblock {\it \bibinfo{journal}{IEEE Transactions on Robotics}\/},  {\it
  \bibinfo{volume}{33}\/}, \bibinfo{pages}{1516--1523}.
\bibitem[{Loizou \& Rimon(2021)}]{loizou2021correct}
\bibinfo{author}{Loizou, S.~G.}, \& \bibinfo{author}{Rimon, E.~D.}
  (\bibinfo{year}{2021}).
\newblock \bibinfo{title}{Correct-by-construction navigation functions with
  application to sensor based robot navigation}.
\newblock {\it \bibinfo{journal}{arXiv preprint arXiv:2103.04445}\/}, .
\bibitem[{Loizou \& Rimon(2022)}]{loizou2022mobile}
\bibinfo{author}{Loizou, S.~G.}, \& \bibinfo{author}{Rimon, E.~D.}
  (\bibinfo{year}{2022}).
\newblock \bibinfo{title}{Mobile robot navigation functions tuned by sensor
  readings in partially known environments}.
\newblock {\it \bibinfo{journal}{IEEE Robotics and Automation Letters}\/},
  {\it \bibinfo{volume}{7}\/}, \bibinfo{pages}{3803--3810}.
\bibitem[{Luo et~al.(2021)Luo, Kantaros \& Zavlanos}]{luo2021abstraction}
\bibinfo{author}{Luo, X.}, \bibinfo{author}{Kantaros, Y.}, \&
  \bibinfo{author}{Zavlanos, M.~M.} (\bibinfo{year}{2021}).
\newblock \bibinfo{title}{An abstraction-free method for multirobot temporal
  logic optimal control synthesis}.
\newblock {\it \bibinfo{journal}{IEEE Transactions on Robotics}\/},  {\it
  \bibinfo{volume}{37}\/}, \bibinfo{pages}{1487--1507}.
\bibitem[{Ogren \& Leonard(2005)}]{ogren2005convergent}
\bibinfo{author}{Ogren, P.}, \& \bibinfo{author}{Leonard, N.~E.}
  (\bibinfo{year}{2005}).
\newblock \bibinfo{title}{A convergent dynamic window approach to obstacle
  avoidance}.
\newblock {\it \bibinfo{journal}{IEEE Transactions on Robotics}\/},  {\it
  \bibinfo{volume}{21}\/}, \bibinfo{pages}{188--195}.
\bibitem[{Otte \& Frazzoli(2016)}]{otte2016rrtx}
\bibinfo{author}{Otte, M.}, \& \bibinfo{author}{Frazzoli, E.}
  (\bibinfo{year}{2016}).
\newblock \bibinfo{title}{Rrtx: Asymptotically optimal single-query
  sampling-based motion planning with quick replanning}.
\newblock {\it \bibinfo{journal}{The International Journal of Robotics
  Research}\/},  {\it \bibinfo{volume}{35}\/}, \bibinfo{pages}{797--822}.
\bibitem[{Panagou(2014)}]{panagou2014motion}
\bibinfo{author}{Panagou, D.} (\bibinfo{year}{2014}).
\newblock \bibinfo{title}{Motion planning and collision avoidance using
  navigation vector fields}.
\newblock In {\it \bibinfo{booktitle}{IEEE International Conference on Robotics
  and Automation (ICRA)}\/} (pp. \bibinfo{pages}{2513--2518}).
\bibitem[{Park \& Kuipers(2015)}]{park2015feedback}
\bibinfo{author}{Park, J.~J.}, \& \bibinfo{author}{Kuipers, B.}
  (\bibinfo{year}{2015}).
\newblock \bibinfo{title}{Feedback motion planning via non-holonomic rrt for
  mobile robots}.
\newblock In {\it \bibinfo{booktitle}{IEEE/RSJ International Conference on
  Intelligent Robots and Systems (IROS)}\/} (pp. \bibinfo{pages}{4035--4040}).
\bibitem[{Qureshi \& Ayaz(2016)}]{qureshi2016potential}
\bibinfo{author}{Qureshi, A.~H.}, \& \bibinfo{author}{Ayaz, Y.}
  (\bibinfo{year}{2016}).
\newblock \bibinfo{title}{Potential functions based sampling heuristic for
  optimal path planning}.
\newblock {\it \bibinfo{journal}{Autonomous Robots}\/},  {\it
  \bibinfo{volume}{40}\/}, \bibinfo{pages}{1079--1093}.
\bibitem[{Rimon(1990)}]{rimon1990exact}
\bibinfo{author}{Rimon, E.} (\bibinfo{year}{1990}).
\newblock {\it \bibinfo{title}{Exact robot navigation using artificial
  potential functions}\/}.
\newblock Ph.D. thesis Yale University.
\bibitem[{Rousseas et~al.(2021)Rousseas, Bechlioulis \&
  Kyriakopoulos}]{rousseas2021harmonic}
\bibinfo{author}{Rousseas, P.}, \bibinfo{author}{Bechlioulis, C.}, \&
  \bibinfo{author}{Kyriakopoulos, K.~J.} (\bibinfo{year}{2021}).
\newblock \bibinfo{title}{Harmonic-based optimal motion planning in constrained
  workspaces using reinforcement learning}.
\newblock {\it \bibinfo{journal}{IEEE Robotics and Automation Letters}\/},
  {\it \bibinfo{volume}{6}\/}, \bibinfo{pages}{2005--2011}.
\bibitem[{Rousseas et~al.(2022{\natexlab{a}})Rousseas, Bechlioulis \&
  Kyriakopoulos}]{rousseas2022optimal}
\bibinfo{author}{Rousseas, P.}, \bibinfo{author}{Bechlioulis, C.~P.}, \&
  \bibinfo{author}{Kyriakopoulos, K.~J.} (\bibinfo{year}{2022}{\natexlab{a}}).
\newblock \bibinfo{title}{Optimal motion planning in unknown workspaces using
  integral reinforcement learning}.
\newblock {\it \bibinfo{journal}{IEEE Robotics and Automation Letters}\/},
  {\it \bibinfo{volume}{7}\/}, \bibinfo{pages}{6926--6933}.
\bibitem[{Rousseas et~al.(2022{\natexlab{b}})Rousseas, Bechlioulis \&
  Kyriakopoulos}]{rousseas2022trajectory}
\bibinfo{author}{Rousseas, P.}, \bibinfo{author}{Bechlioulis, C.~P.}, \&
  \bibinfo{author}{Kyriakopoulos, K.~J.} (\bibinfo{year}{2022}{\natexlab{b}}).
\newblock \bibinfo{title}{Trajectory planning in unknown 2d workspaces: A
  smooth, reactive, harmonics-based approach}.
\newblock {\it \bibinfo{journal}{IEEE Robotics and Automation Letters}\/},
  {\it \bibinfo{volume}{7}\/}, \bibinfo{pages}{1992--1999}.
\bibitem[{S{\'a}nchez et~al.(2021)S{\'a}nchez, D’Jorge, Raffo, Gonz{\'a}lez
  \& Ferramosca}]{sanchez2021nonlinear}
\bibinfo{author}{S{\'a}nchez, I.}, \bibinfo{author}{D’Jorge, A.},
  \bibinfo{author}{Raffo, G.~V.}, \bibinfo{author}{Gonz{\'a}lez, A.~H.}, \&
  \bibinfo{author}{Ferramosca, A.} (\bibinfo{year}{2021}).
\newblock \bibinfo{title}{Nonlinear model predictive path following controller
  with obstacle avoidance}.
\newblock {\it \bibinfo{journal}{Journal of Intelligent \& Robotic Systems}\/},
   {\it \bibinfo{volume}{102}\/}, \bibinfo{pages}{1--18}.
\bibitem[{Shen et~al.(2021)Shen, Wilson, Harvey \& Gupta}]{shen2021smarrt}
\bibinfo{author}{Shen, Z.}, \bibinfo{author}{Wilson, J.},
  \bibinfo{author}{Harvey, R.}, \& \bibinfo{author}{Gupta, S.}
  (\bibinfo{year}{2021}).
\newblock \bibinfo{title}{Smarrt: Self-repairing motion-reactive anytime rrt
  for dynamic environments}.
\newblock {\it \bibinfo{journal}{arXiv preprint arXiv:2109.05043}\/}, .
\bibitem[{Tahir et~al.(2018)Tahir, Qureshi, Ayaz \&
  Nawaz}]{tahir2018potentially}
\bibinfo{author}{Tahir, Z.}, \bibinfo{author}{Qureshi, A.~H.},
  \bibinfo{author}{Ayaz, Y.}, \& \bibinfo{author}{Nawaz, R.}
  (\bibinfo{year}{2018}).
\newblock \bibinfo{title}{Potentially guided bidirectionalized rrt* for fast
  optimal path planning in cluttered environments}.
\newblock {\it \bibinfo{journal}{Robotics and Autonomous Systems}\/},  {\it
  \bibinfo{volume}{108}\/}, \bibinfo{pages}{13--27}.
\bibitem[{Valbuena \& Tanner(2012)}]{valbuena2012hybrid}
\bibinfo{author}{Valbuena, L.}, \& \bibinfo{author}{Tanner, H.~G.}
  (\bibinfo{year}{2012}).
\newblock \bibinfo{title}{Hybrid potential field based control of differential
  drive mobile robots}.
\newblock {\it \bibinfo{journal}{Journal of Intelligent \& Robotic systems}\/},
   {\it \bibinfo{volume}{68}\/}, \bibinfo{pages}{307--322}.
\bibitem[{Vasilopoulos et~al.(2022)Vasilopoulos, Pavlakos, Schmeckpeper,
  Daniilidis \& Koditschek}]{vasilopoulos2022reactive}
\bibinfo{author}{Vasilopoulos, V.}, \bibinfo{author}{Pavlakos, G.},
  \bibinfo{author}{Schmeckpeper, K.}, \bibinfo{author}{Daniilidis, K.}, \&
  \bibinfo{author}{Koditschek, D.~E.} (\bibinfo{year}{2022}).
\newblock \bibinfo{title}{Reactive navigation in partially familiar planar
  environments using semantic perceptual feedback}.
\newblock {\it \bibinfo{journal}{The International Journal of Robotics
  Research}\/},  {\it \bibinfo{volume}{41}\/}, \bibinfo{pages}{85--126}.
\bibitem[{Vlantis et~al.(2018)Vlantis, Vrohidis, Bechlioulis \&
  Kyriakopoulos}]{vlantis2018robot}
\bibinfo{author}{Vlantis, P.}, \bibinfo{author}{Vrohidis, C.},
  \bibinfo{author}{Bechlioulis, C.~P.}, \& \bibinfo{author}{Kyriakopoulos,
  K.~J.} (\bibinfo{year}{2018}).
\newblock \bibinfo{title}{Robot navigation in complex workspaces using harmonic
  maps}.
\newblock In {\it \bibinfo{booktitle}{IEEE International Conference on Robotics
  and Automation (ICRA)}\/} (pp. \bibinfo{pages}{1726--1731}).
\bibitem[{Wang(2023)}]{wang2023harmonic}
\bibinfo{author}{Wang, S.} (\bibinfo{year}{2023}).
\newblock \bibinfo{title}{Harmonic tree}.
\newblock
  \bibinfo{howpublished}{\url{https://shuaikang-wang.github.io/HarmonicTree/}}.
\bibitem[{Warren(1989)}]{warren1989global}
\bibinfo{author}{Warren, C.~W.} (\bibinfo{year}{1989}).
\newblock \bibinfo{title}{Global path planning using artificial potential
  fields}.
\newblock In {\it \bibinfo{booktitle}{IEEE International Conference on Robotics
  and Automation (ICRA)}\/} (pp. \bibinfo{pages}{316--317}).
\bibitem[{Yu et~al.(2015)Yu, Li, Chen \& Allg{\"o}wer}]{yu2015nonlinear}
\bibinfo{author}{Yu, S.}, \bibinfo{author}{Li, X.}, \bibinfo{author}{Chen, H.},
  \& \bibinfo{author}{Allg{\"o}wer, F.} (\bibinfo{year}{2015}).
\newblock \bibinfo{title}{Nonlinear model predictive control for path following
  problems}.
\newblock {\it \bibinfo{journal}{International Journal of Robust and Nonlinear
  Control}\/},  {\it \bibinfo{volume}{25}\/}, \bibinfo{pages}{1168--1182}.

\end{thebibliography}

\end{document}